\documentclass[sigconf]{aamas} 

\usepackage{balance} 

\usepackage{soul}
\usepackage{url}
\usepackage{graphicx}
\usepackage{amsmath}
\usepackage{amsthm}
\usepackage{booktabs}
\usepackage{algorithm}
\usepackage{courier}
\usepackage{xcolor}
\usepackage{float}
\usepackage{tikz}
\usepackage{subfig}
\usepackage{graphicx}
\usepackage{natbib}
\usepackage{microtype}
\usepackage{graphicx}
\usepackage{booktabs} 
\usepackage{url}  
\usepackage{hyperref}
\usepackage{mathrsfs} 
\usepackage{algorithm}
\usepackage{algpseudocode}
\usepackage{natbib}

\usepackage{amsmath}
\usepackage{stackengine}
\usepackage{nccmath}
\usepackage{amsthm}
\usepackage{float}
\usepackage{caption}
\usepackage{subfig}

\usepackage{tikz}
\usetikzlibrary{positioning, shapes, patterns}

\usepackage{comment}

\usepackage{xcolor}

\usepackage{algorithm,algpseudocode}

\newtheorem{defn}{Definition}
\newtheorem{theorem}{Theorem}
\newtheorem{theorem2}{Theorem}
\newtheorem{lemm}{Lemma}
\newtheorem{assumption}{Assumption}

\def\delequal{\mathrel{\ensurestackMath{\stackon[1pt]{=}{\scriptstyle\Delta}}}}

\DeclareMathOperator{\E}{\mathbb{E}}
\DeclareMathOperator*{\argmax}{arg\,max}

\usepackage{amsfonts}       
\usepackage{nicefrac}       

\setcopyright{ifaamas}
\acmConference[AAMAS '23]{Proc.\@ of the 22nd International Conference
on Autonomous Agents and Multiagent Systems (AAMAS 2023)}{May 29 -- June 2, 2023}
{London, United Kingdom}{A.~Ricci, W.~Yeoh, N.~Agmon, B.~An (eds.)}
\copyrightyear{2023}
\acmYear{2023}
\acmDOI{}
\acmPrice{}
\acmISBN{}

\acmSubmissionID{174}

\title[Learning from Multiple Advisors in MARL]{Learning from Multiple Independent Advisors in Multi-agent Reinforcement Learning}

\author{Sriram Ganapathi Subramanian}
\affiliation{
  \institution{Vector Institute, Toronto, Canada}
  \city{University of Waterloo, Waterloo}
  \country{Canada}}
\email{sriram.subramanian@vectorinstitute.ai}

\author{Matthew E. Taylor}
\affiliation{
  \institution{University of Alberta, Edmonton, Canada}
  \city{Alberta Machine Intelligence Institute, Edmonton}
  \country{Canada}}
\email{matthew.e.taylor@ualberta.ca}

\author{Kate Larson}
\affiliation{
  \institution{University of Waterloo}
  \city{Waterloo}
  \country{Canada}}
\email{kate.larson@uwaterloo.ca}

\author{Mark Crowley}
\affiliation{
  \institution{University of Waterloo}
  \city{Waterloo}
  \country{Canada}}
\email{mcrowley@uwaterloo.ca}

\begin{abstract}
Multi-agent reinforcement learning typically suffers from the problem of sample inefficiency, where learning suitable policies involves the use of many data samples. Learning from external demonstrators is a possible solution that mitigates this problem. However, most prior approaches in this area assume the presence of a single demonstrator. Leveraging multiple knowledge sources (i.e., \emph{advisors}) with expertise in distinct aspects of the environment could substantially speed up learning in complex environments.
This paper considers the problem of simultaneously learning from multiple independent advisors in multi-agent reinforcement learning. The approach leverages a two-level $Q$-learning architecture, and extends this framework from single-agent to multi-agent settings. We provide principled algorithms that incorporate a set of advisors by both evaluating the advisors at each state and subsequently using the advisors to guide action selection. We also provide theoretical convergence and sample complexity guarantees. Experimentally, we validate our approach in three different test-beds and show that our algorithms give better performances than baselines, can effectively integrate the combined expertise of different advisors, and learn to ignore bad advice.
\end{abstract}


\newcommand{\BibTeX}{\rm B\kern-.05em{\sc i\kern-.025em b}\kern-.08em\TeX}

\begin{document}


\pagestyle{fancy}
\fancyhead{}


\maketitle 


\section{Introduction}

Reinforcement learning (RL) has been successful in obtaining super-human performances in a wide range of challenges such as Atari games \citep{mnih2015human}, Go \citep{silver2016mastering}, and simple robotic tasks like opening doors and learning visuomotor policies \citep{levine2016end}. However, it has not been straightforward to replicate these successes in complex real-world problems. One reason is that these problems often have a multi-agent structure, where more than one learning agent participates at the same time, resulting in complicated dynamics. Despite research advances in multi-agent reinforcement learning (MARL) \citep{hernandez2019survey}, poor sample efficiency in existing algorithms is one issue that still causes significant hurdles in applying MARL to complex problems \citep{da2019survey}.

Using external sources of knowledge that help in accelerating MARL training is one solution \citep{barrett2017making}, which has extensive support in literature \cite{da2019survey}. However, most prior work include two limiting assumptions. First, all demonstrations need to come from a single demonstrator \citep{Brys2015reinforcement}. In complex MARL environments, since agents learn policies that meet the twin goals of responding to changing opponent(s) and environments \cite{littman1994markov}, a learner can likely benefit from multiple knowledge sources that have expertise in different parts of the environment or different aspects of the task. Second, all demonstrations are near-optimal (i.e., from an ``expert'')~\citep{piot2014boosted}. In practice, these knowledge sources are typically sub-optimal, and we broadly refer to them as \emph{advisors} (to differentiate from experts).

In this paper, we provide an approach that simultaneously leverages multiple different (sub-optimal) advisors for MARL training. Since the advisors may provide conflicting advice in different states, an algorithm needs to resolve such conflicts to take advantage of all the advisors effectively. We propose a two-level learning architecture and formulate a $Q$-learning algorithm for simultaneously incorporating multiple advisors in MARL, improving upon the previous work of Li et al.~\citep{li2019two} in single-agent RL. 
This architecture uses one level to evaluate advisors and the other learns values for actions.
Further, we extend our approach to an actor-critic variant that applies to the centralized training and decentralized execution (CTDE) setting \citep{lowe2017multi}.
Since RL is a fixed point iterative method \citep{szepesvari1999unified}, we provide convergence results, proving that our $Q$-learning algorithm converges to a Nash equilibrium \citep{nash1951non} (under common assumptions). Additionally, we provide a detailed finite-time analysis of our $Q$-learning algorithm under two different types of learning rates.
Finally, we experimentally study our approach in three different multi-agent test-beds, in relation to standard baselines.

Since we relax the two limiting assumptions regarding learning from demonstrators in MARL, our hope is that this approach will spur successes in real-world applications, such as autonomous driving \citep{helou2021reasonable} and fighting wildfires \citep{jain2020review}, where MARL methods could use existing (sub-optimal) solutions as advisors to accelerate training.


\section{Related Work}\label{sec:related}

This work is most related to the approach of \emph{reinforcement learning from expert demonstrations} (RLED) \citep{piot2014boosted}. A well-known RLED technique is \emph{deep $Q$-learning from demonstrations} (DQfD)~\citep{hester2018deep}, which combines a temporal difference (TD) loss, an L2 regularization loss, and a classification loss that encourages actions to be close to that of the demonstrator. 
Another method, \textit{normalized actor-critic} (NAC) \citep{jing2020reinforcement}, drops the classification loss and is more robust under imperfect demonstrations. However, NAC is prone to weaker performances than DQfD under good demonstrations due to the absence of classification loss. A different approach, \emph{human agent transfer} (HAT) \cite{taylor2011integrating}, extracts information from limited demonstrations using a classifier, while \emph{confidence-based human-agent transfer} (CHAT) \citep{wang2017improving} improves HAT by using a confidence measurement to safeguard against sub-optimal demonstrations. A related approach is the teacher-student framework \cite{Torrey2013teaching}, where a pretrained policy (teacher) can be used to provide limited advice to a learning agent (student). Subsequent works expand this framework towards interactive learning \citep{Amir2016Interactive}, however, almost all works in this area assume a moderate level expertise for the teacher. 
Moreover, these are all independent methods primarily suited for single-agent environments, and may not be directly applicable in MARL context.

Furthermore, external knowledge sources have also been used in MARL \citep{da2019survey}, where prior works often assume near optimal experts\citep{reddy2012, waugh2013computational} or are only applicable to restrictive settings, such as fully cooperative or zero-sum competitive games  \citep{omidshafiei2019learning,leno2017simultaneously, silver2016mastering, Wang2018efficient, dayong2020differential}. Leno et al.~\cite{leno2017simultaneously} introduced a framework where an agent can learn from its peers in a shared learning environment, in addition to learning from the environmental rewards. Here the peers can be sub-optimal, however this work only applies to cooperative environments. 
Other works have provided a cooperative teaching framework for hierarchical learning~\cite{Kim2020Learning, yang2021efficient}. For multi-agent general-sum environments, \emph{advising multiple intelligent reinforcement agents - decision making} (ADMIRAL-DM) \cite{Subramanian2022multiagent} is a $Q$-learning approach that incorporates real-time information from a single online sub-optimal advisor. 

One limitation of many prior works is the assumption of a single source of demonstration.
In MARL, it may be possible to obtain advisors from different sources of knowledge that provide conflicting advice. For single-agent settings, Li et al.~\cite{li2019two} provides the two-level $Q$-learning (TLQL) algorithm that incorporates multiple advisors in RL. The TLQL maintains two $Q$-networks, where the first $Q$-network (high-level) keeps track of each advisor's performance and the second $Q$-network (low-level) learns the quality of each action. 
We improve upon TLQL and make it applicable to MARL settings.

\section{Background}\label{sec:background}

\textbf{Stochastic Games:} A $N$-player stochastic game is represented by a tuple $\langle S, A^1, \ldots, A^N, r^1, \ldots, r^N, P, \gamma \rangle$, where $S$ is the state space, $A^j$ is the action space of the agent $j \in \{1, \ldots, N\}$, and $r^j: S \times A^1 \times \cdots \times A^N \xrightarrow{} \mathcal{R}$ is the reward function of $j$. Also, $P: S \times A^1 \times \cdots A^N \xrightarrow{} \Omega(S)$ is the transition probability that determines the next state given the current state and the joint action of all agents, where $\Omega$ is a probability distribution. 
Finally, $\gamma \in [0,1)$ is the discount factor. At each time $t$, all agents observe the global state $s$ and take a local action $a^j$ \citep{shapley1953stochastic}. The joint action $\boldsymbol{a} = \{a^1, \ldots, a^N \}$ determines the immediate reward $r^j$ for $j$ and the next state of the system $s'$. Each agent learns a suitable policy that gives the best responses to its opponent(s). The policy is denoted by $\pi^j: S \xrightarrow{} \Omega(A^j)$.
Let $ \boldsymbol{\pi} \triangleq ( \pi^1, \ldots, \pi^N)$ be the joint policy of all agents. At a state $s$, the value function of $j$ under the joint policy $\boldsymbol{\pi}$ is $v^j_{\boldsymbol{\pi}} (s) = \sum_{t=0}^{\infty} \gamma^t \E_{\pi, P} [r^j_t| s_0=s, \boldsymbol{\pi}] $. This represents the expected discounted future reward of $j$, when all agents follow the policy $\boldsymbol{\pi}$ from the state $s$. Related to the value function, is the action-value function or the $Q$-function. The $Q$-function of agent $j$, under the policy $\boldsymbol{\pi}$, is given by, $Q^j_{\boldsymbol{\pi}} (s, \boldsymbol{a}) = r^j(s, \boldsymbol{a}) + \gamma \E_{s' \sim P} [v^j_{\boldsymbol{\pi}}(s')] $. 

The setting we consider is general-sum stochastic games, where the reward functions of the different agents can be related in any arbitrary fashion. In this setting, the \emph{Nash equilibrium} is typically considered as the solution concept \cite{hu2003nash}, where the joint policy $\boldsymbol{\pi}_* = [\pi^1_*, \ldots, \pi^N_*]$ for all $s \in S$ and all $j$ satisfies $v^j(s; \pi^j_*, \boldsymbol{\pi}^{-j}_{*} ) \geq v^j(s; \pi^j, \boldsymbol{\pi}^{-j}_{*}) $. Here, $\boldsymbol{\pi}^{-j}_{*} \triangleq [\pi^1_{*}, \ldots, \pi^{j-1}_*, \pi^{j+1}_*, \ldots, \pi^N_{*} ]$ represents the joint policy of all agents except $j$. In a Nash equilibrium, each agent plays the best response to the other agents and any deviation from this response is guaranteed to be worse off. 
Further, Hu and Wellman~\cite{hu2003nash} proved that 
the $Q$-updates of an agent $j$, using the Nash payoff at each stage eventually converges to its Nash $Q$ value ($Q^j_*$), which is the action-value obtained by the agent $j$ when all agents follow the joint Nash equilibrium policy for infinite periods.

\textbf{Two-level $Q$-learning}: The TLQL algorithm \cite{li2019two} enables single-agent learning under the simultaneous presence of multiple advisors providing conflicting demonstrations.
Here, the challenge is to determine which advisor to trust in a given state.  
In this regard, the TLQL contains two $Q$-tables, a high-level $Q$-table (abbreviated as high-$Q$) and a low-level $Q$-table (abbreviated as low-$Q$). The high-$Q$ stores the value of the $\langle s, ad \rangle$ pair, where $ad \in AD$ represents an advisor (with $AD$ representing the set of all advisors). 
The high-$Q$ also stores the value of following the RL policy in addition to each advisor. The low-$Q$ maintains the value of each state-action pair. 


At each time step, the agent observes the state and selects an advisor (or the RL policy) from the high-$Q$ using the $\epsilon$-greedy strategy. If the high-$Q$ returns an advisor, then the advisor's recommended action is performed. If the RL policy is returned, then an action is executed from the low-$Q$ based on the $\epsilon$-greedy strategy \cite{sutton1998introduction}. 
The low-$Q$ is updated using the vanilla $Q$-learning Bellman update~\citep{watkins1992q}. Subsequently, the high-$Q$ is updated using a \textbf{synchronization} step. In this step, when an advisor's action is performed, the value of the advisor in the high-$Q$ is simply assigned the value of that action from the low-$Q$. Finally, the high-$Q$ of the RL policy is updated using the relation $highQ(s, RL) = \max_a lowQ(s,a)$. This synchronization update of high-$Q$ preserves the convergence guarantees, due to the policy improvement guarantee in single-agent $Q$-learning \citep{sutton1998introduction}. 

There are two important limitations of TLQL. First, the high-$Q$ that represents the value of the advisors also depends on the RL policy through the synchronization step. This $Q$ value represents the value of taking the action suggested by the advisor at the current state and then following the RL policy from the next state onward. This definition is problematic since at the beginning of training, the RL policy is sub-optimal, and the objective is to accelerate learning by relying on external advisors and avoid using the RL policy at all. As advisors are evaluated at each state using the RL policy, it is likely that the most effective advisor among the set of advisors is not being followed until the RL policy improves. At this stage, it might be possible to simply follow the RL policy itself, defeating the purpose of learning from advisors. Second, the advisors have not been evaluated at the beginning of learning. Hence, it is impossible to find the most suitable advisor to follow, from the available advisors. While TLQL simply follows an $\epsilon$-greedy exploration strategy, this approach could take many data samples to figure out the right advisor. We address both these limitations.

\section{Two-level Architecture in MARL}\label{sec:extending}

We consider a general-sum stochastic game, where there are a set of agents that are learning a policy with an objective of providing a best response to the other agents as well as the environment. Each agent $j$ can access a finite set of (possibly sub-optimal) independent advisors $AD^j$. We use $ad^j$ to represent an advisor of $j$, where $ad^j \in \{ ad^j_1, \ldots, ad_{|AD^j|}^j\}$. 
Each advisor $ad^j$ can be queried by $j$ to obtain an action recommendation at each state of the stochastic game. These online advisors provide real-time action advice to the agent, which helps in learning to dynamically adapt to opponents. 
We consider a centralized training setting and assume 1) the advisors are fixed and do not learn, 2) the communication between agents and their advisors is free, 3) there is no communication directly between learning agents, 4) the environment is fully observable (i.e., an agent can observe the global state, all actions, and all rewards), and 5) the state and action spaces are discrete. Though we require these assumptions for theoretical guarantees, we will show that it is possible to relax a number of these assumptions in practice. 

To make TLQL applicable to multi-agent settings, we parameterize both the $Q$-functions with the joint actions, as is common in practice \citep{littman1994markov}. Also, we do not maintain the RL policy in the high-$Q$ table and do not perform a synchronization step. These steps are no longer needed to preserve the convergence results in multi-agent settings, since we do not have a policy improvement guarantee (unlike in single-agent settings) \cite{tan1993multi}. Instead, we choose to use the probabilistic policy reuse (PPR) technique~\cite{Fern2006probabilistic}, where a hyperparameter ($\epsilon' \in [0,1]$) decides the probability of following any advisor(s) (i.e., using the high-$Q$) or the agents' own policy (i.e., using the low-$Q$) for action selection, at each time step during training. This hyperparameter starts with a high value (maximum dependence on the available advisor(s)) at the beginning of training and is decayed (linearly) over time. After some finite time step during training, the value of this hyperparameter goes to 0 (no further dependence on any advisor(s)) and the agent only uses its low-$Q$ (own policy) for action selection. This helps in two ways: 1) in the time limit ($t \xrightarrow{} \infty$), a learning agent has the possibility of recovering from poor advising (by learning from the environment), and 2) eventually the trained agent can be independently deployed (with no requirement of having access to any advisor(s)). 

The general structure of our proposed \emph{Multi-Agent Two-Level $Q$-Learning} (MA-TLQL) algorithm is given in Figure~\ref{fig:matlqlstructure}. Since we are in a fully observable setting, like \cite{hu2003nash}, we specify that each agent maintains copies of the $Q$-tables of other agents from which it can obtain the joint actions of other agents for the current state. If such copies cannot be maintained, agents could use the previously observed actions of other agents for the joint action as done in prior works \citep{Subramanian2022multiagent, pmlr-v80-yang18d}. We use the two-level architecture, where each agent will maintain a high-$Q$ as well as a low-$Q$. The high-$Q$ provides a value for the $\langle s, \boldsymbol{a}^{-j}, ad^j \rangle$ tuples, where $\boldsymbol{a}^{-j} = \{a^1, \ldots,a^{j-1}, a^{j+1}, \ldots, a^N\}$ is the joint action of all agents except the agent $j$. 
This high-$Q$ is a value estimate for the advisor $ad^j$ as estimated by the agent $j$ at the state $s$ and joint action $\boldsymbol{a}^{-j}$.
The high-$Q$ estimates are updated with an evaluation update given by
\begin{equation}\label{eq:evaluation}
\begin{array}{lll}
    highQ^j_{t+1}(s, \boldsymbol{a}^{-j}, ad^j)   = highQ^j_t(s, \boldsymbol{a}^{-j}, ad^j) 
     \\ \quad + \alpha \Big( (r^j_t + \gamma highQ^j_t(s', \boldsymbol{a'}^{-j}, ad^j) - highQ^j_t(s, \boldsymbol{a}^{-j}, ad^j) \Big),
\end{array}
\end{equation}
\noindent where $s$ and $s'$ are the states at $t$ and $t+1$, and $\alpha$ is the learning rate. Also, $\boldsymbol{a}^{-j}$ and $\boldsymbol{a'}^{-j}$ are joint actions at $s$ and $s'$, respectively.

\begin{figure}
    \centering
    \includegraphics[width=0.45\textwidth]{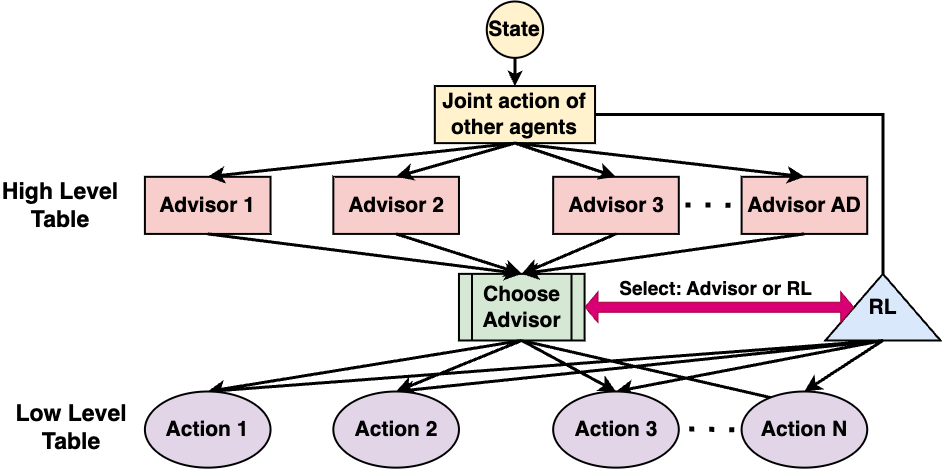}
    \caption{Structure of MA-TLQL, for a representative agent having access to a set of AD advisor(s)}
    \label{fig:matlqlstructure}
\end{figure}

As described previously, a hyperparameter is used to decide between choosing to follow an advisor or the RL policy. If the agent follows an advisor, the high-$Q$ is used to select an advisor using an ensemble selection technique. Let us denote, $\mathcal{Q}^j = \{highQ^{j}(s, \boldsymbol{a}^{-j}, ad_1^j), \ldots, highQ^{j}(s, \boldsymbol{a}^{-j}, ad_m^j)\}$, to represent the high-$Q$ estimates of a set of $M$ advisors (with $|M| = m$) advising the same action $a^j$ to an agent $j$. Here, $ad^j_i$ represents an advisor $i \in \{1, \ldots, m\}$ of $j$. Then the value of vote for action $a^j$, at the state $s$ and the joint action $\boldsymbol{a}^{-j}$, denoted by $\mathcal{V}^j(s, \boldsymbol{a}^{-j}, a^j)$, is calculated as
\begin{equation}\label{eq:valueofvote}
    \begin{array}{l}
         \mathcal{V}^j(s, \boldsymbol{a}^{-j}, a^j) = \max \mathcal{Q}^j \\ + \sum_{i = 1, i \neq \argmax_i highQ^j(s, \boldsymbol{a}^{-j}, ad^j_i)}^{m} \frac{1}{\mu(s)} highQ^j(s, \boldsymbol{a}^{-j}, ad^j_i).
    \end{array}
\end{equation}
\noindent Here, $\mu(s)$ is the number of times the agent has visited the state $s$. In Eq.~\ref{eq:valueofvote}, if an action is recommended by more than one advisor, the value of its vote is a weighted sum of all high-$Q$ estimates of advisors recommending that action. Each high-$Q$ estimate (except the best high-$Q$ estimate) is weighted by the reciprocal of the number of times the respective state is visited. In this way, when a state is visited many times, the advisor with the best high-$Q$ estimate is likely to be followed (wisdom of individual). When a state is visited only a few times, then the action suggested by a majority of advisors is likely to be selected (wisdom of crowd). From Eq.~\ref{eq:valueofvote}, if an action $a^j$ is recommended by only one advisor, then the value of vote for $a^j$ will be equal to the high-$Q$ estimate of that advisor. 
After the value of votes for all actions are calculated, the action with the maximum value of vote is executed, and the high-$Q$ estimate of the advisor recommending this action is updated by the agent $j$ using Eq.~\ref{eq:evaluation}. 

If the agent decides to use its RL policy, it uses its low-$Q$, which contains a value for the $\langle s, \boldsymbol{a}^{-j}, a^j \rangle$ tuples (value for each action). 
At each step, the low-$Q$ is updated using a control update as follows:

\begin{equation}\label{eq:control}
\begin{array}{ll}
    lowQ^j_{t+1}(s, \boldsymbol{a}^{-j}, a^j ) =  lowQ^j_t(s, \boldsymbol{a}^{-j}, a^j) 
    \\  + \alpha \big(r^j_t +  \gamma \max_{a'^j} lowQ^j_t(s', \boldsymbol{a}'^{-j}, a'^{j}) - lowQ^j_t(s, \boldsymbol{a}^{-j}, a^j ) \big).
    \end{array}
\end{equation}


Now we describe how MA-TLQL addresses the two limitations of TLQL. The first was the dependence of high-$Q$ on the RL policy in TLQL. Note, the high-$Q$ in MA-TLQL maintains the $Q$ values of the advisor themselves, i.e., the value of following the advisor's policy from the current state onward (see Eq.~\ref{eq:evaluation}). Thus, the coupling between the advisor values and the RL policy is removed (no synchronization). 
 The second was the difficulty in picking the right advisor in TLQL. MA-TLQL uses an ensemble technique to choose the advisor during the early stages of learning. In later stages, it switches to following the best advisor according to the high-$Q$ estimates, which addresses this limitation of TLQL. In Appendix~\ref{sec:illustrativeexample}, we present a toy example that illustrates the limitations of TLQL. 


We provide the complete pseudocode for a tabular implementation of the MA-TLQL algorithm in Appendix~\ref{sec:psuedocode} (Algorithm~\ref{alg:twolevelqlearning}). Further, we extend this approach to large state-action environments using a neural network based implementation (Algorithm~\ref{alg:twolevelqlearningNN}), which uses a target network and a replay buffer, as in the Deep $Q$-learning (DQN) algorithm \citep{mnih2015human}. We also provide an actor-critic implementation (Algorithm~\ref{alg:twolevelactorcritic}) which is suitable for CTDE \citep{lowe2017multi}. We will refer to this algorithm as \emph{multi-agent two-level actor-critic} (MA-TLAC). In MA-TLAC, each agent has two actors and two critics (high-level and low-level), where the respective $Q$-functions serve as the critic and the corresponding policies serve as the actors. 
In this CTDE method, agents can obtain global information (including actions and rewards of other agents) during training, however, the agents only require access to its local observation during execution. This makes our method applicable to partially observable environments as in Lowe et al.~\cite{lowe2017multi}. MA-TLAC applies to continuous state space environments as well (refer to Appendix~\ref{sec:psuedocode} for more details).



\section{Theoretical Results}\label{sec:theoriticalresults}

We present a convergence guarantee for tabular MA-TLQL and characterize the convergence rate. For these results, we build on some prior works that provide several fundamental results on the nature of stochastic iterative functions \cite{berksekas1996neuro, Dar2003learning}. We apply these to MA-TLQL in general-sum stochastic games using three assumptions from Hu and Wellman~\cite{hu2003nash}, where 
the first two are standard \citep{szepesvari1999unified}. 

\begin{assumption}\label{assumption:visitassumption}
    Every $s\in \mathcal{S}$ and $a^j \in A^j$, for every agent $j$
    are visited infinitely often, and the reward function ($\forall j$) stays bounded. 
\end{assumption}

\begin{assumption}\label{assumption:learningrate}
For all $s, t$, and $\boldsymbol{a}$, $0 \leq \alpha_t(s,\boldsymbol{a}) < 1$, $\sum_{t=0}^\infty \alpha_t(s, \boldsymbol{a}) = \infty$, $\sum_{t=0}^\infty[\alpha_t(s, \boldsymbol{a})]^2 < \infty$.
\end{assumption}

\begin{assumption}\label{assumption:globaloptimum}
The Nash equilibrium is a global optimum or saddle point in every stage game of the stochastic game. 
\end{assumption}

The third assumption is a restriction on the nature of the stochastic game. Several prior works note that this assumption is restrictive but needed to theoretically prove the convergence of $Q$-learning methods in general-sum stochastic games with two or more agents. In practice, however, it is still possible to observe convergence of $Q$-learning methods when this assumption is violated \citep{hu2003nash, Subramanian2022multiagent, pmlr-v80-yang18d}.

Now we prove our theoretical results. All theorem statements are provided here, while the proofs can be found in Appendices~\ref{sec:convergence} -- \ref{sec:linaerlearningrate}. First, we provide the convergence guarantee for the low-$Q$. Recall, the PPR technique guarantees that the MA-TLQL dependence on high-$Q$ is only until a finite time step during training. After this step, the agent only uses its low-$Q$ for action selection. As the convergence result in Theorem~\ref{theorem:lowqconvergence} is provided in the time limit ($t \xrightarrow{} \infty$), the influence of high-$Q$ can be neglected for this result. 

\begin{theorem}\label{theorem:lowqconvergence}
Given Assumptions~\ref{assumption:visitassumption}, \ref{assumption:learningrate}, \ref{assumption:globaloptimum}, the low-$Q$ values of an agent $j$ converges to its Nash $Q$ value in the limit ($t \xrightarrow{} \infty$). 
\end{theorem}

Next, we provide sample complexity bounds for the MA-TLQL algorithm. Instead of explicitly considering the high-$Q$ values, we specify that the underlying joint policy has a covering time of $L$. The covering time specifies an upper bound on the number of time steps needed for all state-joint action pairs to be visited at least once starting from any state-joint action pair. 
Further, since the action selection is only based on the low-$Q$ values in the limit ($t \xrightarrow{} \infty$), we are most interested in the sample complexity of low-$Q$, where the dependence on the high-$Q$ is effectively represented by $L$.

Regarding sample complexity, as is done in \cite{Dar2003learning}, we distinguish between two kinds of learning rates. Consider the following equation for the low-$Q$ (rewriting Eq.~\ref{eq:control} and dropping $low$ for simplicity),

\begin{equation}\label{eq:lowlevelq}
\begin{array}{l}
    Q^j_{t+1}(s_t, \boldsymbol{a}_t) 
    = \big(1-\alpha^\omega_t(s_t, \boldsymbol{a}_t))(Q^j_t(s_t, \boldsymbol{a}_t) \big)  
    \\ \quad \quad \quad \quad \quad \quad + \alpha^\omega_t(s_t, \boldsymbol{a}_t) \big(r^j_t + \gamma \max_{a^j}Q^j_t (s_{t+1}, \boldsymbol{a}_{t+1}) \big).
\end{array}
\end{equation}

\noindent The value of $\alpha^\omega_t (s, \boldsymbol{a})= \frac{1}{[\#(s, \boldsymbol{a}, t)^\omega]}$, where $\#(s, \boldsymbol{a}, t)$ is the number of times until $t$ that the joint action $\boldsymbol{a}$ is performed at $s$. 
Here, we consider $ \omega \in (1/2, 1]$. The learning rate is linear if $\omega = 1$, and the learning rate is polynomial if $\omega \in (1/2, 1)$.

The next theorem provides a lower bound on the number of time steps needed for convergence in the case of a polynomial learning rate. From Assumption~\ref{assumption:visitassumption}, let us specify that all rewards for the agent $j$ are bounded by $R^j_{\max}$. We consider a variable $Q^j_{\max}$, which denotes the maximum possible low-$Q$ value for the agent $j$, which is bounded by $Q^j_{\max} = R^j_{\max}/(1-\gamma)$. Additionally, we also use another variable $\beta = (1 - \gamma)/2$ to present our upcoming results concisely.

\begin{theorem}\label{theorem:lowqsamplecomplexity}
Let us specify that with probability at least $1 - \delta$, for an agent $j$, $||Q^j_T - Q^j_* ||_\infty \leq \epsilon$. The bound on the rate of convergence of low-$Q$, $Q^j_T$, with a polynomial learning rate of factor $\omega$ is given by (with $Q^j_*$ as the Nash $Q$-value of the agent $j$)

\begin{equation}
\begin{array}{l}
      T = \Omega \Big( \Big(\frac{L^{1 + 3\omega}Q^{2,j}_{\max}\ln(\frac{|S|\Pi_i |A_i| Q^j_{\max}}{\delta \beta \epsilon})}{\beta^2 \epsilon^2} \Big)^{1 - \omega}/L \\ \quad \quad \quad \quad \quad \quad \quad \quad \quad \quad \quad + \Big( (\frac{L}{\beta} \ln \frac{Q^j_{\max}}{\epsilon}  + 1)/2 \Big)^{\frac{1}{1-\omega}}\Big). 
    \end{array}
\end{equation}

\end{theorem}

Assuming the same action spaces for all agents (i.e. $|A_1| = |A_2| = \cdots = |A_N| = |A|$), we note that the dependence on the number of agents is $\ln |A|^N = N \ln|A|$. Overall this results in a sub-linear dependence on the number of agents based on the value of $\omega$, which is far superior to recent works that report an exponential dependence on the number of agents when learning in general-sum stochastic game environments (with an arbitrary number of agents) for convergence to a Nash equilibrium \citep{Qinghua2021sharp,song2021can}.
Further, the dependence on the state space and action space in Theorem~\ref{theorem:lowqsamplecomplexity} is sub-linear ($\ln |S|$), and the dependence on the covering time is $\Omega(L^{2\omega-3\omega^2} + L^{1/1-\omega})$, which is a polynomial dependence.

The next theorem considers the linear learning rate case.

\begin{theorem}\label{theorem:linearlearningrate}
Let us specify that with probability at least $1 - \delta$, for an agent $j$, $||Q^j_T - Q^j_* ||_\infty \leq \epsilon$. The bound on the rate of convergence of low-$Q$, $Q^j_T$, with a linear learning rate is given by
\begin{equation}
    T = \Omega \Big( (L + \psi L + 1)^{\frac{1}{\beta} \ln \frac{Q^j_{\max}}{\epsilon}} \frac{Q^{2,j}_{\max}\ln(\frac{|S|\Pi_i |A_i| Q^j_{\max}}{\delta \beta \epsilon \psi})}{\beta^2 \epsilon^2 \psi^2}   \Big),
\end{equation}
\noindent where $\psi$ is a small arbitrary positive constant satisfying $\psi \leq 0.712$. 
\end{theorem}

Theorem~\ref{theorem:linearlearningrate} shows that the bound is linear in the number of agents and sub-linear in the state and action spaces. This linear dependence on the number of agents is also superior to prior results~\citep{Qinghua2021sharp,song2021can}. Note, the dependence on the covering time in Theorem~\ref{theorem:linearlearningrate} could be much worse than that of Theorem~\ref{theorem:lowqsamplecomplexity}, depending on the value of $Q^j_{\max}$ and $\epsilon$. Since the value of $\epsilon$ is small, the dependence is certainly worse than that obtained for the polynomial learning rate case. Also, the dependence on $Q^j_{\max}$ is exponential as opposed to a polynomial dependence for Theorem~\ref{theorem:lowqsamplecomplexity}. 
The last two theorems illustrate the performance benefit in using a polynomial learning rate as opposed to a linear learning rate in our algorithm.  

\section{Experiments and Results}\label{sec:experiments}

We consider three different experimental domains, one each for competitive, cooperative, and mixed settings, where each agent has access to a set of four advisors. We use neural network implementations of MA-TLQL and MA-TLAC, along with 5 other baselines: DQN \citep{mnih2015human}, DQfD \citep{hester2018deep}, CHAT \citep{wang2017improving}, ADMIRAL-DM \citep{Subramanian2022multiagent}, and TLQL \citep{li2019two}. In Appendix~\ref{sec:natureofalgorithms}, we tabulate the characteristics of these baselines and provide further details regarding our choices. Since CHAT and ADMIRAL-DM assume the presence of a single advisor, we use a weighted random policy approach for implementing these two algorithms in the multiple-advisor setting, as in Li et al. \cite{li2019two}. If different advisors provide different actions at the same state, each action is weighted based on the number of advisors suggesting that action.
For DQfD, during pre-training \citep{hester2018deep}, we populate the replay buffer using advisor demonstrations from all the available advisors. For all our experiments, we will describe the critical details here, while the complete description is in Appendix~\ref{sec:experimentaldetails}. All the experiments are repeated 30 times, with averages and standard deviations reported. For statistical significance we use the unpaired 2-sided t-test and report $p$-values, where $p<0.05$ is considered significant. The tests compare the highest performing algorithm (typically MA-TLQL) with the second-best baseline and best/average advisor performance. We conduct a total of seven experiments. The code for all experiments is open-sourced \cite{sourcecode}. Appendix~\ref{sec:experimentaldetails} tabulates all our experimental settings. Appendix~\ref{sec:hyperparameters} provides the hyperparameter details and Appendix~\ref{sec:wallclocktimes} contains the wall clock times.


Experiments 1--4 use the competitive, two-agent version of Pommerman \citep{resnick2018pommerman}.  
The environment is complex, with each state containing roughly 200 elements related to agent position and special features (e.g., bombs). The reward function is sparse: agents only receive a terminal reward of $\{-1, 0, +1\}$. Experiments are conducted 
in two phases. In the first phase (training), our algorithms and the baselines train against a standard DQN opponent for 50,000 episodes, where we plot the cumulative rewards. During this phase, algorithms can use advisors to accelerate training. In the second phase (execution), we test the performance of the trained policies against DQN for 1000 episodes, where we plot the win rate (fraction of games won) for each algorithm. During this phase, agents cannot access advisors, take no exploratory actions, and do not learn. All advisors pertaining to these four experiments are rule-based agents.


\begin{figure}[h!]
	\subfloat[Training ]{{\includegraphics[width=0.44\textwidth, height=3cm]{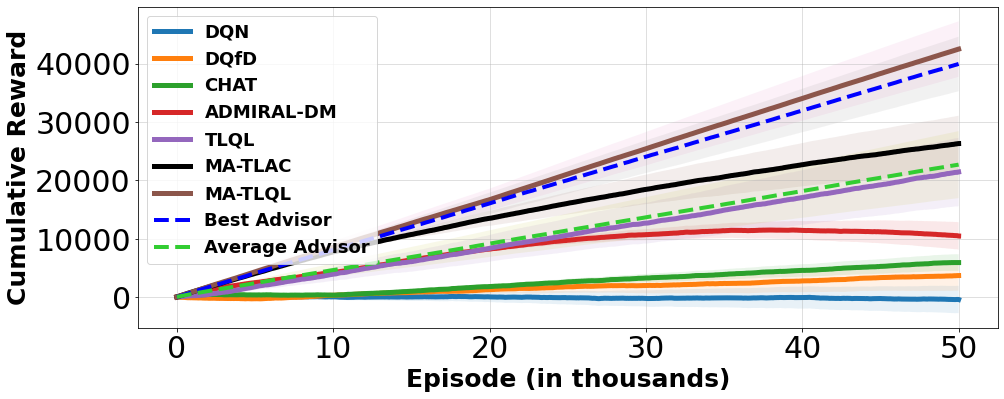}}} 
	\\ 
	\subfloat[Execution]{{\includegraphics[width=0.44\textwidth, height=3cm]{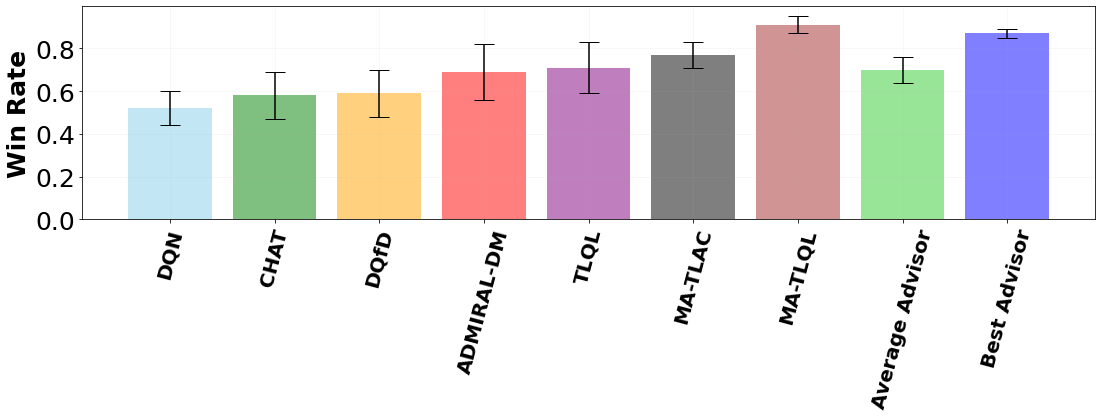}}}

  \caption{Two agent Pommerman with four sufficient advisors of different quality (Experiment 1)}%
	\label{fig:onevsonesufficientdifferent}
\end{figure}

\textbf{Experiment 1:} Our first experiment uses a set of four advisors ranked in terms of quality from Advisor~1 to Advisor~4. Here, Advisor~1 is the best advisor, capable of teaching the agent all skills 
needed to win the game of Pommerman, and Advisor~4 only suggests random actions. In Pommerman, there is a fixed set of six skills that an agent needs to master to be able to win \cite{resnick2018pommerman}. Since this set of advisors can teach all these skills, we say the agent has access to a \emph{sufficient set} of advisors.
We plot the training and execution performances in Figure~\ref{fig:onevsonesufficientdifferent}(a) and (b) respectively, 
including the performance of the best and average advisors (average of all Advisors~1--4) against DQN. MA-TLQL gives the best performance ($p < 0.01$) and is the only algorithm providing a better performance than the best advisor ($p < 0.11$) in both training and execution. MA-TLAC
performs better than the average advisor ($p < 0.04$). None of the others show better performances than the average advisor. CHAT and ADMIRAL-DM are not capable of leveraging and distinguishing amongst a set of advisors.
DQfD uses pre-training, which is not very effective in the non-stationary multi-agent context. Learning from online advising is preferable in MARL. Also, DQfD and CHAT are independent techniques that are not actively tracking the opponent's performance. While TLQL is capable of learning from multiple advisors, its independent nature in addition to coupling of advisor values with the RL policy reduces its effectiveness in multi-agent environments. MA-TLQL gives a better performance than MA-TLAC in both training and execution ($p < 0.01$). As noted previously, the $Q$-learning family of algorithms tends to induce a positive bias while using the maximum action value, which leads to providing the best possible response \cite{Hasselt2010double}. This explains the superior performance of MA-TLQL. 
We conclude that MA-TLQL is capable of leveraging a set of good and bad advisors. Further, the training results in Figure~\ref{fig:onevsonesufficientdifferent}(a) show that MA-TLQL is able to learn a better policy faster than the baselines by using advisors ($p < 0.01$). The evaluation results in Figure~\ref{fig:onevsonesufficientdifferent}(b) show that amongst all algorithms trained for the same number of episodes, MA-TLQL provides the best performance, when deployed without any advisors ($p < 0.01$). Both observations point to better sample efficiency in MA-TLQL.    
Supplementary experiments in Appendix~\ref{sec:frequency} show that MA-TLQL comes to relying more on good advisors than poor advisors, as compared to the baselines, illustrating its superiority.

\begin{figure}[h!]
	\subfloat[Training]{{\includegraphics[width=0.44\textwidth, height=3cm]{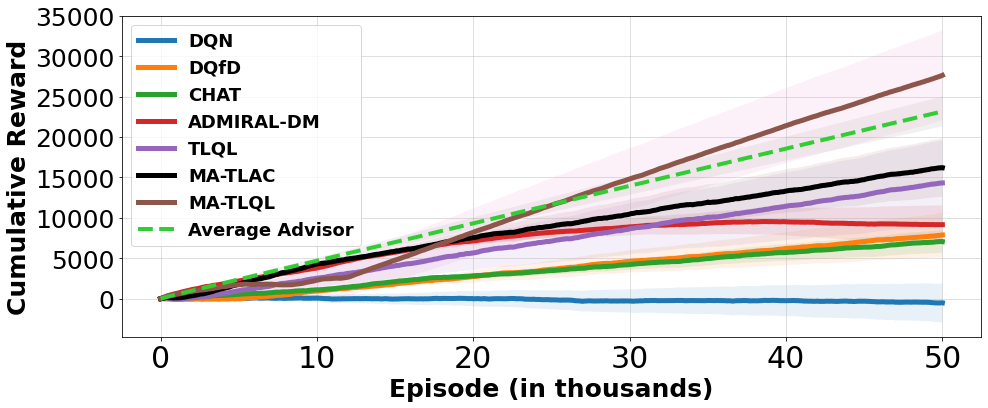}}}
    \\
	\subfloat[Execution]{{\includegraphics[width=0.44\textwidth, height=3cm]{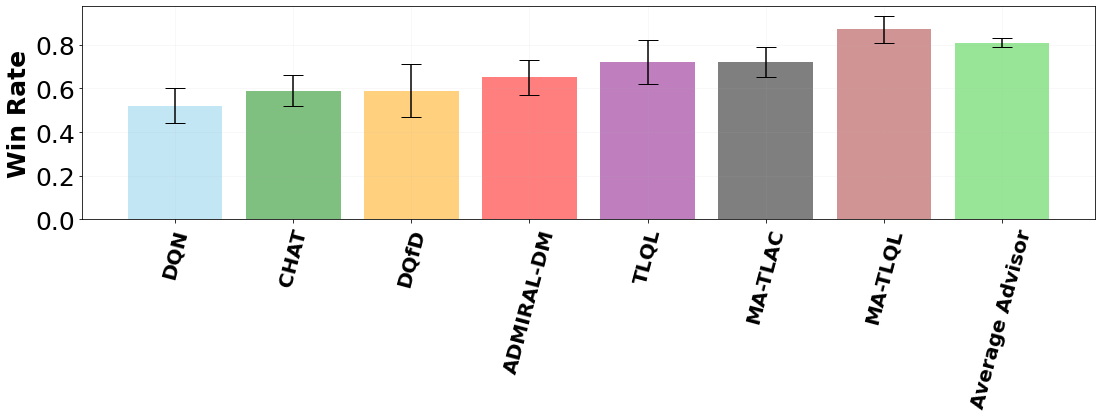}}}

  \caption{Two-agent Pommerman with four sufficient advisors of similar quality (Experiment 2)}%
	\label{fig:onevsonesufficientsimilar}
\end{figure}

\textbf{Experiment 2:} We use the same domain as in Experiment 1, but with a different set of advisors. Now, all four advisors can teach strictly different Pommerman skills. 
For example, Advisor~1 can teach how to escape the enemy (and nothing else), and Advisor~2 can teach how to obtain necessary power-ups (and nothing else --- full details are in Appendix~\ref{sec:experimentaldetails}). These advisors provide psuedo-random action advice in states outside their expertise. This set of advisors is also a sufficient set. Now, learning agents must decide what advisor to listen to in the current state.
From the training and execution results in Figure~\ref{fig:onevsonesufficientsimilar}(a) and (b), we see that MA-TLQL gives the best overall performance ($p < 0.02$), exceeding the average performance of the four advisors ($p < 0.05$). Since all four advisors have similar quality, we only choose to use the average performance of the four advisors in this experiment for comparison. 
We conclude that MA-TLQL is capable of leveraging the combined knowledge of a set of advisors with different individual expertise, during learning. 

\begin{figure}[h!]
	\subfloat[Training]{{\includegraphics[width=0.44\textwidth, height=3cm]{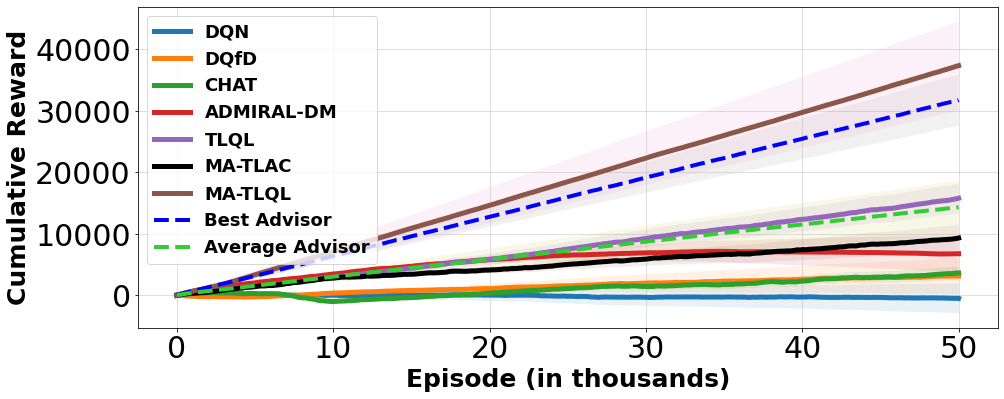}}}
	\\
			\subfloat[Execution ]{{\includegraphics[width=0.44\textwidth, height=3cm]{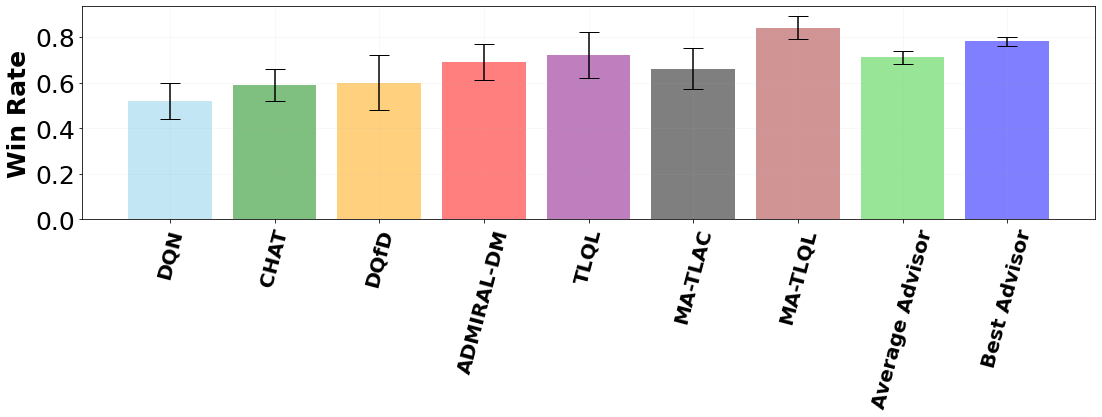}}}

  \caption{Two-agent Pommerman with four insufficient advisors of different quality (Experiment 3)}%
	\label{fig:onevsoneinsufficientdifferent}
\end{figure}

\textbf{Experiment 3:} We use the same domain as in Experiment~1 but with a different set of four advisors. These advisors are similar to the set of advisors in our first experiment, where Advisor~1 gives the best advice throughout the domain, and Advisor~4 is random. However, this set of advisors is \emph{not} capable of teaching all the strategies (i.e, Pommerman skills) needed to win in Pommerman, and compose an insufficient set (more details in Appendix~\ref{sec:experimentaldetails}). It is critical for agents to learn from the environment in addition to the advisors. Training and execution results in Figure~\ref{fig:onevsoneinsufficientdifferent} shows the superior performance of MA-TLQL, the only algorithm that outperforms the best advisor ($p < 0.05$) and all baselines ($p < 0.02$). Surprisingly, TLQL performs better than MA-TLAC ($p < 0.02$), likely due to the positive bias of $Q$-learning. This experiment reinforces the observation that MA-TLQL is capable of learning from good advisors and avoids bad advisors (also see Appendix~\ref{sec:frequency}). 
Since MA-TLQL outperforms the best advisor, this experiment demonstrates that MA-TLQL can learn from both, advisors and through direct interactions with the environment, hence having a much improved sample efficiency as compared to other algorithms that learn only from the environment. This is observed during both training and execution.

\begin{figure}[h!]
    \subfloat[Training]{{\includegraphics[width=0.44\textwidth, height=3cm]{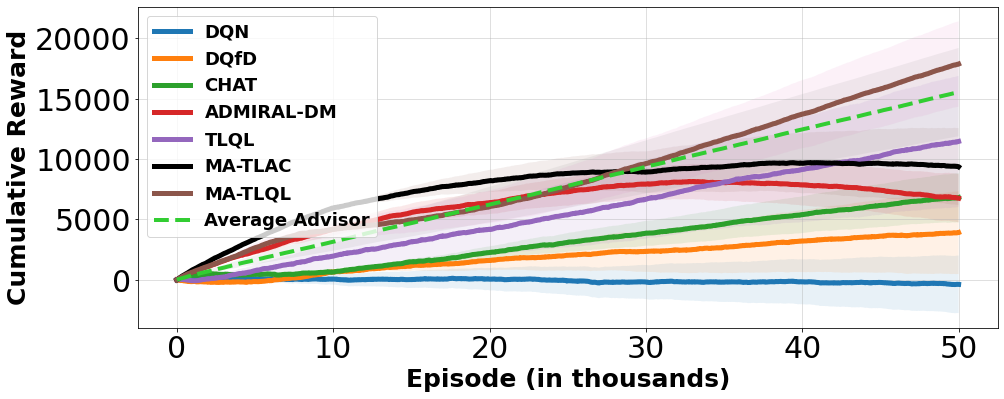}}}
		\\
		\subfloat[Execution]{{\includegraphics[width=0.44\textwidth, height=3cm]{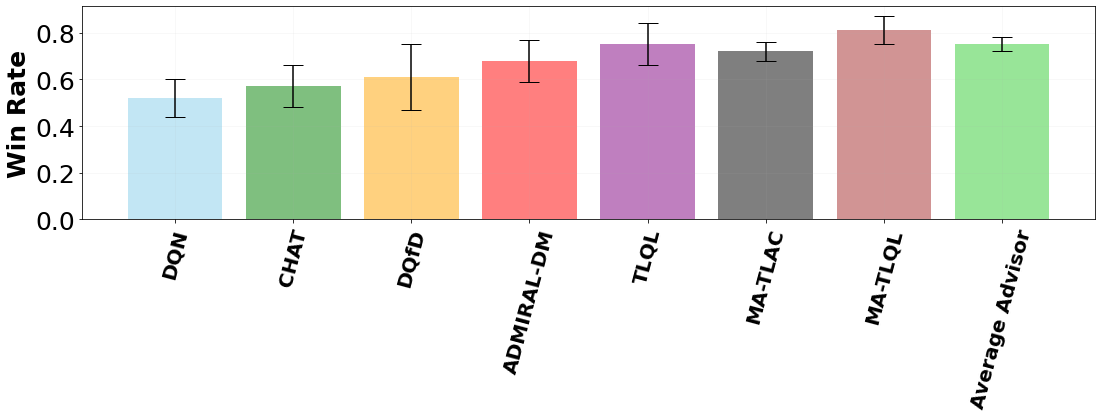}}}

  \caption{Two-agent Pommerman with four insufficient advisors of similar quality (Experiment 4)}%
	\label{fig:onevsoneinsufficientsimilar}
\end{figure}

\textbf{Experiment 4:} This is similar to the Experiment~2: four advisors have similar quality, but each understands a different Pommerman skill. However, our set of advisors in this experiment are insufficient to teach all the skills in Pommerman, and the agent must also learn from the environment.
The results in Figure~\ref{fig:onevsoneinsufficientsimilar} shows that MA-TLQL is capable of leveraging the combined expertise of the advisors and learning from the environment to obtain the best performance, as compared to the baselines ($p < 0.04$) and advisors ($p < 0.05$). This makes MA-TLQL more sample efficient than the prior algorithms.

\begin{figure}[h!]
	\centering
	\subfloat[ Training ]{{\includegraphics[width=0.44\textwidth, height=3cm]{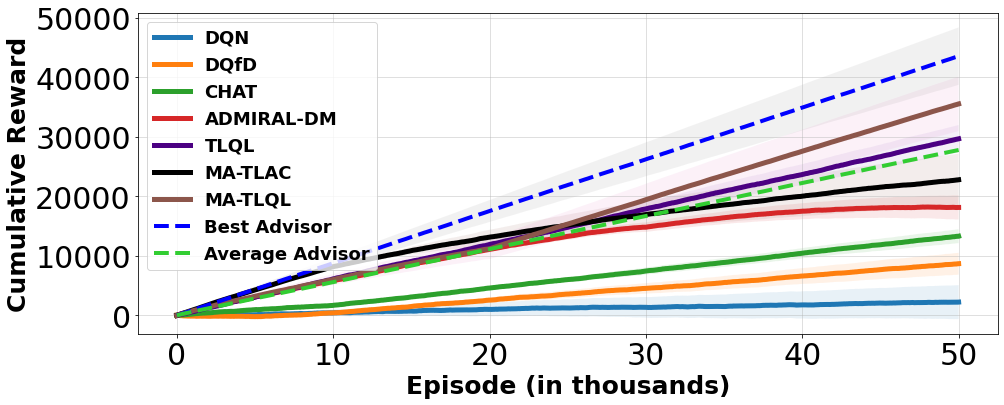}}}
    \\
	\subfloat[ Execution ]{{\includegraphics[width=0.44\textwidth, height=3cm]{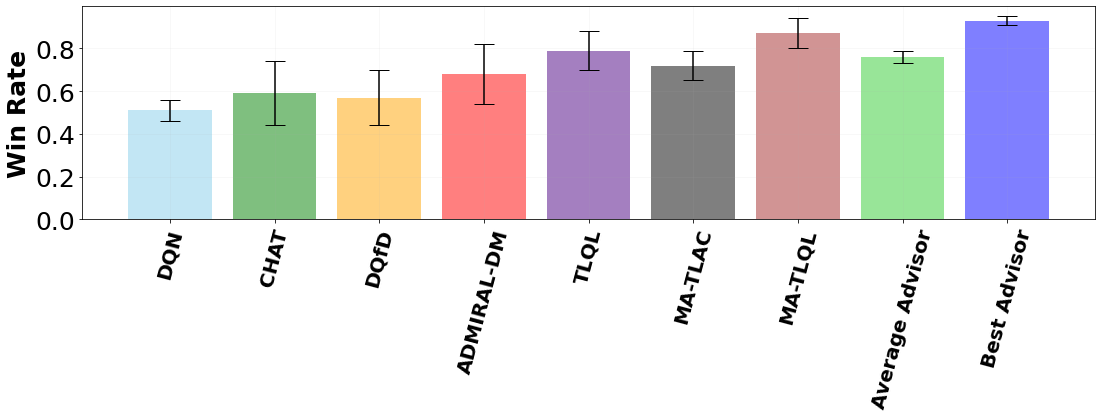}}}

  \caption{Team (mixed) Pommerman (Experiment 5)}%
	\label{fig:teamsufficientdifferent}
\end{figure}

\textbf{Experiment 5:} We now switch to a four-agent version of Pommerman, which is two vs.~two. This is a mixed setting as agents need to learn cooperative as well as competitive skills. Overall, this is a more complex domain with a larger state space. We consider four sufficient advisors of different quality, similar to Experiment 1. We conduct two phases --- training (for 50,000 episodes) and execution (for 1000 episodes). 
The training and execution results in Figure~\ref{fig:teamsufficientdifferent} show that MA-TLQL provides the best performance compared to the baselines ($p < 0.04$) but does not perform better than the best available advisor. Since this is a more complex domain, MA-TLQL needs a larger training period for learning good policies. However, MA-TLQL still performs better than the average performance of the four advisors ($p < 0.03$). 
We conclude that although MA-TLQL's performance 
suffers in the more difficult mixed setting, it still outperforms all the other baselines and is capable of distinguishing between good and bad advisors (see also Appendix~\ref{sec:frequency}). From both training and execution results in Figure~\ref{fig:teamsufficientdifferent}, we note that MA-TLQL has a superior sample efficiency as compared to the other baselines.

\begin{figure}[h!]
		\subfloat[Training ]{{\includegraphics[width=0.22\textwidth, height = 3cm]{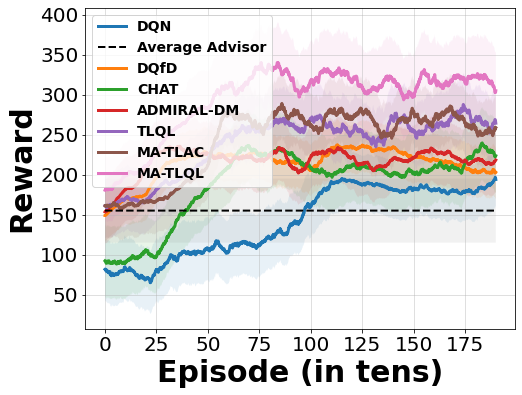}}}
		\quad
		\subfloat[Execution ]{{\includegraphics[width=0.22\textwidth, height = 3cm]{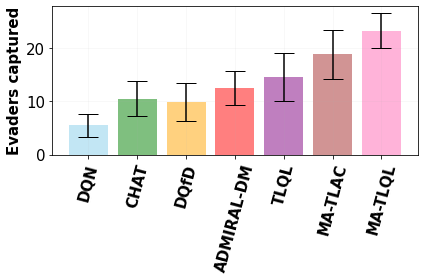}}}
  \caption{Cooperative Pursuit setting (Experiment 6)}%
	\label{fig:pursuitcooperation}
\end{figure}

\textbf{Experiment 6:} This experiment switches to the cooperative Pursuit domain \cite{gupta2017cooperative}. There are eight pursuer learning agents that learn to capture a set of 30 randomly moving targets (evaders) (details in Appendix~\ref{sec:experimentaldetails}). We use four pre-trained DQN networks as the advisors, learning for 500, 1000, 1500, and 2000 episodes, respectively. We again have two phases --- training and execution. During training, all algorithms are trained for 2000 episodes. The trained networks are then used in the execution phase for 100 episodes with no further training or influence from advisors. 
Figure~\ref{fig:pursuitcooperation}(a) plots the episodic rewards obtained during training and the Figure~\ref{fig:pursuitcooperation}(b) plots the number of targets captured in the execution phase, where MA-TLQL shows the best performance ($p < 0.03$). 
Hence, MA-TLQL can outperform all baselines in a cooperative environment as well.

\begin{figure}[h!]
		\subfloat[Training ]{{\includegraphics[width=0.22\textwidth, height = 3cm]{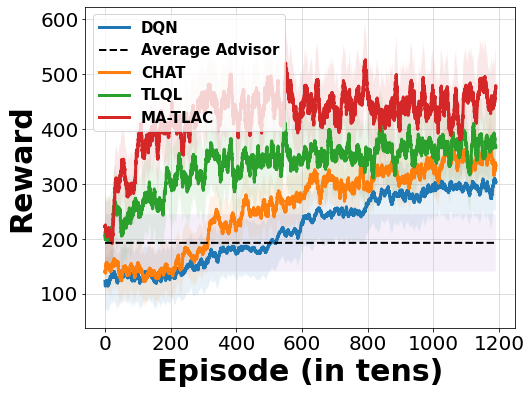}}}
		\quad 
		\subfloat[Execution ]{{\includegraphics[width=0.22\textwidth, height = 3cm]{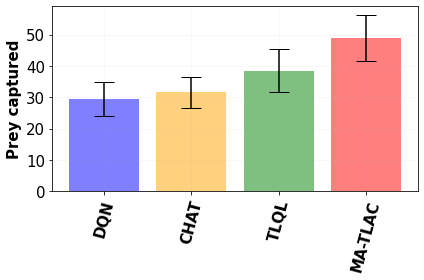}}}
  \caption{Mixed Predator-Prey setting (Experiment 7)}%
	\label{fig:mperesults}
\end{figure}

\textbf{Experiment 7:} This final experiment considers a mixed cooperative competitive Predator-Prey environment which is a part of the Multi Particle Environment (MPE) suite \cite{lowe2017multi}. Our implementation uses a discrete action space and a continuous state space (more details in Appendix~\ref{sec:experimentaldetails}). There are a total of eight predators trying to capture eight prey (prey are not removed, but respawned upon capture). In our experiment, each algorithm trains the predators while the prey is trained using a standard DQN opponent. The experiments have two phases of training and execution, which is modelled as a CTDE setting. Here each agent obtains information about the actions and rewards of all other agents during training, but only has local observation during execution. Since this environment requires decentralization during execution, we omit the fully centralized MA-TLQL and ADMIRAL-DM. We also omit DQfD since it gave poor performances previously. As in Experiment~6, we use four pretrained DQN (predator) networks as advisors (trained for 1000, 2000, 7000, and 12000 episodes). Training is conducted for 12000 episodes and execution is conducted for 100 episodes. The training results in Figure~\ref{fig:mperesults}(a) (plot of episodic rewards) show that MA-TLAC is the most sample efficient compared to other algorithms as it is able to leverage the available advisors better than others, thus outperforming them ($p < 0.04$). The execution results in Figure~\ref{fig:mperesults}(b) plots the average prey captured by each algorithm. MA-TLAC outperforms others during execution as well ($p < 0.03$). 

From all the p-values across the seven experiments, we note that most of our observations are statistically significant. Despite observing MA-TLQL outperforming the best advisor in many of the experiments, some of these comparisons are not statistically significant (i.e., $p \geq 0.05$). While the main experiments of the paper consider fixed advisors, our algorithms can also be implemented with learning/changing advisors (see Appendix~\ref{sec:learningadvisors}). In Appendix~\ref{sec:differentadvisors} we study performances under different numbers of advisors. Also, our algorithms can be used along with opponent modelling techniques as done by prior works \cite{he2016opponent} (more details in Appendix~\ref{sec:opponentmodelling}). 

\section{Ablation Study}

In this section, we run an ablation study on the three components of MA-TLQL that differ from the previously introduced TLQL algorithm by Li et al.~\cite{li2019two}. To recall these three components are: i) joint action (JA) updates, ii) ensemble method (EM), and iii) advisor evaluation (AE). For this ablation study we will consider the two-agent version of Pommerman with four sufficient advisors having different (Experiment~1) and similar quality (Experiment~2).

\begin{figure}[t]
	\subfloat[Training ]{{\includegraphics[width=0.22\textwidth, height=3cm]{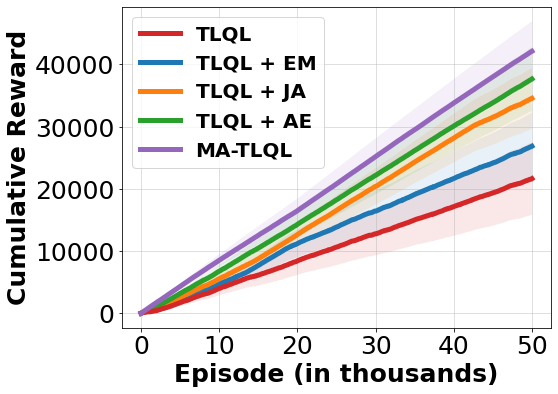}}} \quad
			\subfloat[Execution]{{\includegraphics[width=0.22\textwidth, height=3cm]{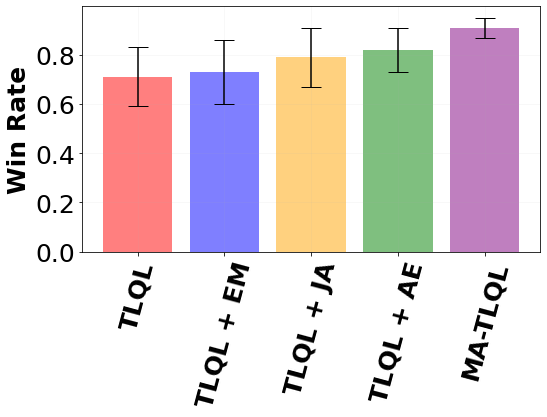}}}

  \caption{Ablation results using Experiment~1}%
	\label{fig:ablationdifferent}
\end{figure}

The ablation results corresponding to Experiment~1 are given in Figure~\ref{fig:ablationdifferent}, where we plot the performances of TLQL and MA-TLQL in addition to TLQL with each of the three components. In Figure~\ref{fig:ablationdifferent}(a) and (b), the performance of TLQL with each of the three components is better than vanilla TLQL. TLQL using the ensemble method (i.e., TLQL+EM) is able to perform better than vanilla TLQL, since at the beginning of training the $Q$-values of the advisors are not accurate, and the ensemble technique chooses the advisor action that is agreed upon by most advisors in the given set (in line with our discussions in Section~\ref{sec:extending}). Recall that the set of four different advisors had four advisors of decreasing quality, with the first three advisors capable of teaching some useful Pommerman skills and the last advisor being just random (see Appendix~\ref{sec:experimentaldetails}). Using the ensemble prevents the use of the random advisor, as the first three advisors are more likely to agree upon an action, increasing the possibility of the agent choosing that action. Further, we see that TLQL highly benefits from using the joint action update (i.e., TLQL+JA) instead of an independent update seen in vanilla TLQL. The joint action update explicitly considers the strategies of other agent(s) and helps in providing stronger best responses as compared to an independent update in the multi-agent environments. Finally, TLQL using advisor evaluation in the high-$Q$ table (i.e., TLQL+AE) provides the best benefit compared to the other components. As discussed in Section~\ref{sec:background} and Section~\ref{sec:extending}, the high-$Q$ definition in vanilla TLQL is limiting since the advisor evaluation through the high-$Q$ is coupled with the inaccurate RL policy (and AE addresses this limitation). Further, from Figure~\ref{fig:ablationdifferent}, we see that MA-TLQL (integrating all the three components) shows the best performance as compared to vanilla TLQL and individual TLQL implementations with each of the three components ($p$ < 0.05). Thus, MA-TLQL is able to seamlessly integrate the advantages of each of the individual components of TLQL, demonstrating its superiority. 

\begin{figure}[t]
	\subfloat[Training ]{{\includegraphics[width=0.22\textwidth, height = 3cm]{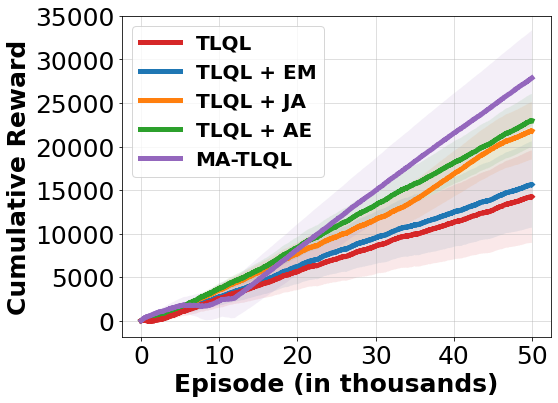}}} \quad \quad  
			\subfloat[Execution]{{\includegraphics[width=0.22\textwidth, height=3cm]{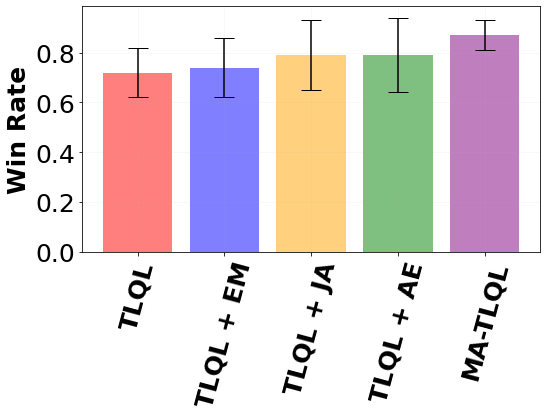}}}

  \caption{Ablation results using Experiment~2}%
	\label{fig:ablationsimilar}
\end{figure}

We also consider a similar ablation study using Experiment~2 (see Figure~\ref{fig:ablationsimilar}). As in Figure~\ref{fig:ablationdifferent}, we see that TLQL with each of the three components performs better than vanilla TLQL. Since we have four advisors of similar quality where each advisor is good at a different Pommerman skill, their agreement on an action is expected to be small. Hence, the ensemble technique (i.e., TLQL+EM) provides only a small improvement over vanilla TLQL. However, the other two components (i.e., TLQL+JA and TLQL+AE) provides a good performance benefit over TLQL. Finally, MA-TLQL, that integrates all the three components, provides the best performance ($p$ < 0.03).

\section{Conclusion}\label{sec:conclusion}

This paper provided a principled approach for learning from multiple independent advisors in MARL. Inspired by Li et al.~\cite{li2019two}, we present a two-level architecture for multi-agent environments. We discuss two limitations in TLQL and address these limitations in our approach. Also, we provide a fixed point guarantee and sample complexity bounds regarding the learning of MA-TLQL. Additionally, we provided an actor-critic implementation that can work in the CTDE paradigm. Further, we performed an extensive experimental analysis of MA-TLQL and MA-TLAC in cooperative, competitive, and mixed settings, where we show that these algorithms are capable of suitably leveraging a set of advisors, and perform better than baselines. 
As future work, we would like to consider human advisors and further explore some avenues in the real-world context.

\section*{Acknowledgements}
Resources used in preparing this research were provided by the province of Ontario and the government of Canada through CIFAR, NSERC and companies sponsoring the Vector Institute. Part of this work has taken place in the Intelligent Robot Learning (IRL) Lab at the University of Alberta, which is supported in part by research grants from the Alberta Machine Intelligence Institute (Amii); a Canada CIFAR AI Chair, Amii; Compute Canada; Huawei; Mitacs; and NSERC.

\bibliographystyle{ACM-Reference-Format}
\bibliography{main}

\clearpage
\newpage 
\appendix

\section{Algorithm Pseudocodes}\label{sec:psuedocode}

A complete pseudocode of a tabular implementation of our $Q$-learning based algorithm (MA-TLQL) is given in Algorithm~\ref{alg:twolevelqlearning}. All agents initialize a low-$Q$ table and a high-$Q$ table in line~2. Then at each state, all agents choose to perform an action in lines~8--20. This action can come from the advisor or the RL policy as described in Section~\ref{sec:extending}. Then the action is executed, and the next state and reward are observed in line~21. Finally, the $Q$ values for the low-$Q$ as well as the high-$Q$ are updated (line~22 and line~23) according to equations presented in Section~\ref{sec:extending}. The value of $\epsilon'$ is linearly decayed from a high-value to a value close to zero during training (line~24). 

To make Algorithm~\ref{alg:twolevelqlearning} applicable to high dimensional state and action spaces, we provide a function approximation-based implementation of MA-TLQL in Algorithm~\ref{alg:twolevelqlearningNN}. Here neural networks are used as the function approximator, and the algorithm uses a separate target network and a replay buffer for training, as introduced in the well-known DQN algorithm \citep{mnih2015human}. The agent maintains a high-$Q$ network and two low-$Q$ networks (evaluation and target networks) and updates these networks using the temporal difference (T.D.) errors with the update equations presented in Section~\ref{sec:extending}. If the full state of the stochastic game is not available, the agent can simply use its observation instead of the state, as applicable in most function approximation-based RL methods. 

We also extend Algorithm~\ref{alg:twolevelqlearningNN} to an actor-critic implementation described in Algorithm~\ref{alg:twolevelactorcritic}. This algorithm is called multi-agent two-level actor-critic (MA-TLAC). This algorithm uses the policy as the actor and the $Q$-values as the critic, consistent with prior work \citep{konda1999actor}. We maintain two actors, and two critics to reflect the two-level (high and low) nature of our algorithm. The high-level actor determines an advisor and the high-level critic helps train the high-level actor, using the T.D. errors. Similarly, the low-level actor determines the appropriate action, with the low-level critic providing the T.D. errors for training. In MA-TLAC, since we use a separate actor network for advisor selection, we do not use the ensemble technique from Eq.~\ref{eq:valueofvote}. Instead, the high-level actor directly chooses one amongst the given advisors for the current state. The advantage of this algorithm is that it can be implemented using the popular CTDE paradigm \citep{lowe2017multi}, since the actors do not require the actions of other agents for action/advisor selection. In CTDE, global information (i.e., information from other agents) is available during training time but not available during execution. This CTDE based implementation also allows our method to be used in partially observable domains, since the actors can use the local observations for action/advisor selection while the critic can use the joint actions and states during training, as described in Lowe et al.~\cite{lowe2017multi}. Also, since the ensemble technique (Eq.~\ref{eq:valueofvote}) is not used in MA-TLAC, it is also applicable to continuous state space environments as well (unlike MA-TLQL, which is only applicable to environments with discrete state spaces).

All the algorithm pseudocode provided in this section assume that all agents are using the same algorithmic steps for learning where it can maintain copies of updates of other agents, as done in prior work \citep{hu2003nash}. If this is not possible, the agents would directly use the observed previous actions of other agents for its updates.

\begin{algorithm}
\caption{MA-TLQL Tabular Method}
\label{alg:twolevelqlearning}
\begin{algorithmic}[1]
\State Let $Ad^j$ denote a set of advisors available to the agent $j$.  
\State For all $j \in {1, \ldots, N}$, $s \in S$, and $a^j \in A^j$: $lowQ^j(s, a^j, \boldsymbol{a}^{-j}) \xleftarrow{} 0$ where $\boldsymbol{a}^{-j} = [a^1, \ldots, a^{j-1}, a^{j+1}, \ldots, a^N]$
\State For all $s \in S$, $ad^j \in Ad^j$, and for all $j \in {1, \ldots, N}$: $highQ^j(s,\boldsymbol{a}^{-j}, ad^j) \xleftarrow{} 0$
\State Initialize a value for hyperparameters $\epsilon$ and $\epsilon'$ and $\eta$
\While{training is not finished}
\State For each agent $j$, get the current state $s$
\State For each agent $j$, get the joint actions of other agents $\boldsymbol{a}^{-j}$ at state $s$ using the respective copies and previous actions of all agents
\State For each agent $j$, let $u$ be a uniform random number between 0 and 1
 
  \If{$u$ < $\epsilon'$}
  
   \State Let $u'$ be a uniform random number between 0 and 1
    \If {$u'$ < $\eta$}
    
    \State Choose an advisor using the high-$Q$ values of agent 
     $j$ for the current state and joint action of other agents from Eq.~\ref{eq:valueofvote} and use its action as the current action $a^{j}_t$ 
    
    \Else
    
    \State Set the advisor $ad^j$ as a random advisor from 
    $Ad^j$ and use its action as the current action $a^{j}_t$.
    \EndIf
    
    \ElsIf{$u$ > $\epsilon'$ and $u$ < $\epsilon$}
    
    \State Set the action $a^{j}_t$ as a random action from the action space $A^j$
    
    \Else
    
    \State Choose a greedy action $a^{j}_t$ from the low-$Q$ value using $s$ and the joint action of other agents 
    
    \EndIf

\State Execute the joint action $\boldsymbol{a}$, observe joint reward $\boldsymbol{r}$ and the next state $s'$, where $\boldsymbol{a} = [a^1, \ldots, a^N]$ and $\boldsymbol{r} = [r^1, \ldots, r^N]$
\State Update value of low-$Q$ for the agent $j$ using Eq.~\ref{eq:control}. Obtain the next actions for other agents $\boldsymbol{a'}^{-j}$ from the respective copies and previous actions of other agents
\State If an advisor was chosen, update value of high-$Q$ of the advisor for the agent $j$ using Eq.~\ref{eq:evaluation}
 \State At the end of each episode, linearly decay $\epsilon'$
\EndWhile
\end{algorithmic}
\end{algorithm}

\begin{algorithm}
\fontsize{9pt}{9pt}\selectfont
\caption{MA-TLQL Neural Network Method}
\label{alg:twolevelqlearningNN}
\begin{algorithmic}[1]
\State Let $Ad^j$ denote a set of advisors available to the agent $j$  
\State Initialize $Q_{\phi^j}, Q_{\phi_{-}^j}$ for all $j \in {1, \ldots, N}$ (to denote low-$Q$). Initialize $Q_{\theta^j}$, $Q_{\theta_{-}^j}$ for all $j \in {1, \ldots, N}$ (to denote high-$Q$)
\State Initialize a value for hyperparameters $\epsilon$ and $\epsilon'$ and $\eta$
\While{training is not finished}
\State For each agent $j$, get the current state $s$
\State For each agent $j$, get the joint actions of other agents $\boldsymbol{a}^{-j}$ at state $s$ using the respective copies and previous actions of all agents
\State For each agent $j$, let $u$ be a uniform random number between 0 and 1

  \If{$u$ < $\epsilon'$}
  
   \State Let $u'$ be a uniform random number between 0 and 1
    \If {$u'$ < $\eta$}
    
    \State Choose an advisor using the high-$Q$ values of agent 
    $j$ from Eq.~\ref{eq:valueofvote} for the current state and joint action of other agents using the high-$Q$, $Q_{\theta^j}$, and use its action as the current action $a^{j}_t$
    
    \Else
    
    \State Set the advisor $ad^j$ as a random advisor from 
    $Ad^j$ and use its action as the current action $a^{j}_t$.
    \EndIf
    
    \ElsIf{$u$ > $\epsilon'$ and $u$ < $\epsilon$}
    
    \State Set the action $a^{j}_t$ as a random action from the action space $A^j$
    
    \Else
    
    \State Choose a greedy action $a^{j}_t$ from the low-$Q$, $Q_{\phi^j}$, using $s$ and the joint action of other agents 
    
    \EndIf
    
\State Execute the joint action $\boldsymbol{a}$, observe joint reward $\boldsymbol{r}$ and the next state $s'$, where $\boldsymbol{a} = [a^1, \ldots, a^N]$ and $\boldsymbol{r} = [r^1, \ldots, r^N]$

\State For each agent $j$, store $ \langle s, \boldsymbol{a}, \boldsymbol{r}, s', \boldsymbol{a'} \rangle$ in replay buffer $\mathcal{D}^j$, where $\boldsymbol{a} = [a^1, \ldots, a^N]$, $\boldsymbol{a'} = [a'^1, \ldots, a'^N]$. Obtain the next actions for other agents $\boldsymbol{a'}^{-j}$ from the respective copies and previous actions of other agents

\State If an advisor was used, for each agent $j$, store $ \langle s, \boldsymbol{a}, \boldsymbol{r}, s', \boldsymbol{a'}, ad^j \rangle$ in replay buffer $\mathcal{D}'^j$, where $ad^j$ is the advisor

\State Set the next state $s'$ as the current state $s$

\State At the end of each episode, linearly decay $\epsilon'$

\While{j = 1 to N}
\State Sample a minibatch of K experiences $ \langle s, \boldsymbol{a}, \boldsymbol{r}, s', \boldsymbol{a'} \rangle$ from $\mathcal{D}^j$
\State Set $y^j = r^j + \gamma \max_{a'^j} Q_{\phi^j_{-}}(s', \boldsymbol{a'}^{-j}, a'^j)$ according to Eq.~\ref{eq:control} 
\State Update the $Q$-network $\phi^j$ by minimizing the loss $\mathcal{L}(\phi^j) = \frac{1}{K} \sum ( y^j  - Q_{\phi^j}(s, \boldsymbol{a}^{-j}, a^j))^2$
\State Sample a minibatch of K experiences $ \langle s, \boldsymbol{a}, \boldsymbol{r}, s', \boldsymbol{a'}, ad^j \rangle$ from $\mathcal{D'}^j$
\State Set $y^j = r^j + \gamma Q_{\theta^j_{-}}(s', \boldsymbol{a'}^{-j}, ad^j)$ according to Eq.~\ref{eq:evaluation}
\State Update the $Q$-network $\theta^j$ by minimizing the loss $\mathcal{L}(\theta^j) = \frac{1}{K}\sum ( y^j  - Q_{\theta^j}(s, \boldsymbol{a}^{-j}, ad^j))^2$
\EndWhile
\State Update the parameters of the target network for each agent by copying over the evaluation network every $\mathcal{T}$ steps: $\phi^j_{-} \xleftarrow{} \phi^j $ and $\theta^j_{-} \xleftarrow{} \theta^j $

\EndWhile
\end{algorithmic}
\end{algorithm}

\begin{algorithm}
\fontsize{9pt}{9pt}\selectfont
\caption{MA-TLAC}
\label{alg:twolevelactorcritic}
\begin{algorithmic}[1]
\State Let $Ad^j$ denote a set of advisors available to the agent $j$ 
\State Initialize $Q_{\phi^j}, \pi_{\phi_{-}^j}$, the low-level critic and actor networks for all $j \in \{1, \ldots, n\}$
\State Initialize $Q_{\theta^j}, \pi_{\theta_{-}^j}$, the high-level critic and actor networks for all $j \in \{1, \ldots, n\}$
\State Initialize a value for hyperparameters $\epsilon$ and $\epsilon'$ and $\eta$
\While{training is not finished}
\State For each agent $j$, get the current state $s$
\State For each agent $j$, let $u$ be a uniform random number between 0 and 1
 
 \If{$u$ < $\epsilon'$}
  
   \State Let $u'$ be a uniform random number between 0 and 1
    \If {$u'$ < $\eta$}
    
    \State Choose an advisor $ad^j$ using the high-level actor $\pi_{\theta^j}$, for the agent $j$, for the current state $s$, and use its action as the current action $a^{j}$
    
    \Else
    
    \State Set the advisor $ad^j$ as a random advisor from $Ad^j$ and use its action as the current action $a^{j}$
    \EndIf
    
    \ElsIf{$u$ > $\epsilon'$ and $u$ < $\epsilon$}
    
    \State Set the action $a^{j}$ as a random action from the action space $A^j$
    
    \Else
    
    \State Choose a greedy action $a^{j}$ from the low-level actor, $\pi_{\phi^j}$, using $s$
    
    \EndIf

\State Execute the joint action $\boldsymbol{a}$, observe joint reward $\boldsymbol{r}$ and the next state $s'$, where $\boldsymbol{a} = [a^1, \ldots, a^N]$ and $\boldsymbol{r} = [r^1, \ldots, r^N]$

\State For each agent $j$, obtain the joint actions of other agents $\boldsymbol{a}^{-j}$ (current observed actions of other agents) at state $s$

\State Set $y^j = r^j + \gamma \max_{a'^j} Q_{\phi^j}(s', \boldsymbol{a'}^{-j}, a'^j)$ according to Eq.~\ref{eq:control}

\State For each $j$, update the low-level critic by minimizing the loss $\mathcal{L}(\phi^j) = (y^j - Q_{\phi^j}(s, \boldsymbol{a}^{-j}, a^j))^2 $

\State For each $j$, calculate the advantage estimate using the relation $A(s,\boldsymbol{a}^{-j}, a^j) = y^j - \sum_{a^j} \pi_{\phi_{-}^j}(a^j|s) Q_{\phi^j}(s, \boldsymbol{a}^{-j}, a^j) $

\State For each $j$, update the low-level actor using the log loss 
$\mathcal{J}(\phi_{-}^j) = \log \pi_{\phi_{-}^j}(a^j|s)A(s, \boldsymbol{a}^{-j}, a^j)$

\State If an agent $j$ used an advisor $ad^j$, then update the advisor's $Q$-estimate.  

\State For each $j$, set $y^j = r^j + \gamma Q_{\theta^j}(s', \boldsymbol{a'}^{-j}, ad^j)$ according to Eq.~\ref{eq:evaluation}

\State Obtain the next actions for other agents $\boldsymbol{a'}^{-j}$ from the respective copies

\State For each $j$, update the high-level critic by minimizing the loss $\mathcal{L}(\theta^j) = (y^j - Q_{\theta^j}(s, \boldsymbol{a}^{-j}, ad^j))^2 $ where $ad^j$ is the advisor chosen by the agent $j$

\State For each $j$, calculate the advantage estimate using the relation $A(s,\boldsymbol{a}^{-j}, ad^j) = y^j - \sum_{ad^j} \pi_{\theta_{-}^j}(ad^j|s) Q_{\theta^j}(s, \boldsymbol{a}^{-j}, ad^j) $

\State For each $j$, update the high-level actor using the log loss 
$\mathcal{J}(\theta_{-}^j) = \log \pi_{\theta_{-}^j}(a^j|s) A(s,\boldsymbol{a}^{-j}, ad^j)$

\State Set the next state as the current state $s = s'$

\State At the end of each episode, linearly decay $\epsilon'$

\EndWhile

\end{algorithmic}
\end{algorithm}
 
\clearpage
\newpage

\section{Proof of Theorem~\ref{theorem:lowqconvergence}}\label{sec:convergence}

\begin{theorem2}
Given Assumptions~\ref{assumption:visitassumption}, \ref{assumption:learningrate}, \ref{assumption:globaloptimum}, the low-$Q$ values of an agent $j$ converges to its Nash $Q$ value in the limit ($t \xrightarrow{} \infty$). 
\end{theorem2}

\begin{proof}

Our proof will be along the lines of Theorem~3 in Subramanian et al.~\cite{Subramanian2022multiagent}.

Let us consider a lemma from prior work.

\begin{lemm}\label{lemma:randomprocess}

A random iterative process 

\begin{equation}\label{eq:deltaeq}
\begin{array}{l}
     \Delta_{t+1}(x) = (1 - \alpha_t(x))\Delta_t(x) + \alpha_t(x) F_t(x)
\end{array}{}
\end{equation}

\noindent where $x \in X$, $t = 0,1, \ldots, \infty$, converges to zero with probability one (w. p. 1) if the following properties hold: 
 
1. The set of possible states $X$ is finite. 

2. $0 \leq \alpha_t(x) \leq 1$, $\sum_t \alpha_t(x) = \infty$, $\sum_t \alpha^2_t(x) < \infty$ w. p. 1, where the probability is over the learning rates $\alpha_t$. 

3. $|| \E \{{F_t(x)|\mathscr{P}_t}\} ||_W \leq \mathscr{K} ||\Delta_t||_W + c_t$, where $\mathscr{K} \in [0,1)$ and $c_t$ converges to zero w. p. 1. 

4. $\textrm{\textbf{var}}\{F_t(x) | \mathscr{P}_t\} \leq K(1 + ||\Delta_t||_W)^2$, where $K$ is some constant. 

\noindent Here $\mathscr{P}_t$ is an increasing sequence of $\sigma$-fields that includes the past of the process.  In particular, we assume that $\alpha_t, \Delta_t, F_{t-1} \in \mathscr{P}_t$. The notation $||\cdot||_W$ refers to some (fixed) weighted maximum norm and the notation $\textrm{\textbf{var}}$ refers to the variance.

\end{lemm}

\begin{proof}

Refer to Theorem~1 in Jaakola et al.~\cite{jaakkola1994convergence} for proof. 

\end{proof}

Across this section, since we are only focusing on the low-$Q$ values, with a small abuse of notation, we will use $Q$ to denote the low-$Q$ values. Now, we define a Nash operator $P_t$, using the following equation, 

\begin{equation}\label{eq:nashoperator}
\begin{array}{l}
    P_t Q^k(s, \boldsymbol{a})
    = \E_{s' \sim p} [r^k_t(s,\boldsymbol{a})  + \gamma \pi^1_{*} (s') \cdots \pi^n_{*}(s') Q^k(s')]
    \end{array}
\end{equation}

\noindent where $s'$ is the state at time $t+1$, $(\pi^1_{*} (s') ,\ldots, \pi^n_{*}(s'))$ is the Nash equilibrium solution for the stage game $(Q^1(s'), \ldots, Q^n(s'))$, and $p$ is the transition function. $Q^k$ denotes the $Q$-value of a representative agent $k$.

\begin{lemm}\label{lemm:nashoperator}
Under Assumption~\ref{assumption:globaloptimum}, the Nash operator as defined in Eq.~\ref{eq:nashoperator} forms a contraction mapping with the fixed point being the Nash $Q$-value of the game. 

\end{lemm}

\begin{proof}
See Theorem~17 of Hu and Wellman~\cite{hu2003nash}. 
\end{proof}

Now, since the $P_t$ operator forms a contraction mapping, $||P_t Q^j - P_t Q^j_*|| \leq \gamma || Q^j - Q^j_*||$, is satisfied for some $\gamma \in [0,1)$ and all $Q^j$. Here $Q^j_*$ is the Nash $Q$-value of the agent $j$.

The objective is to apply Lemma~\ref{lemma:randomprocess} to show that the low-$Q$ in MATLQL converge to the Nash $Q$ values.

The first two conditions of Lemma~\ref{lemma:randomprocess} are satisfied from the Assumption~\ref{assumption:visitassumption} and Assumption~\ref{assumption:learningrate}. Now, comparing Eq.~\ref{eq:deltaeq} and Eq.~\ref{eq:control} we get that $x$ can be associated with the state joint action pairs $(s,\boldsymbol{a}) $ and  $\Delta_t(s_t,a_t)$ can be associated with $Q^j_t(s,\boldsymbol{a}) - Q^j_*(s,\boldsymbol{a})$. Here, $Q^j_*(s,\boldsymbol{a})$ is the Nash $Q$ value of the agent $j$. 

Now we get 

\begin{equation}
    \begin{array}{l}
         \Delta_{t+1}(x) = (1 - \alpha_t(x)) \Delta_t(x) + 
         \alpha_t(x)F_t(x), 
    \end{array}{}
\end{equation}

\noindent where 

\begin{equation}\label{eq:contraction}
    \begin{array}{l}
         F_t(x) = r^j_t + \gamma v^{Nash,j}(s_{t+1}) 
         - Q^j_*(s_t,\boldsymbol{a_t})  
            \\  \quad \quad \quad + \gamma[\max_{a^j}Q^j_t(s_{t+1}, \boldsymbol{a}_{t+1}) - v^{Nash, j}(s_{t+1})] 
         \\ \\
        \delequal r^j_t + \gamma v^{Nash,j}(s_{t+1})  - Q^j_*(s_t,\boldsymbol{a}_t)  +  C^j_t(s_t,\boldsymbol{a}_t)
        \\ \\ 
         \delequal  F_t^{Q,j}(s_t,\boldsymbol{a}_t) + C^j_t(s_t,\boldsymbol{a}_t) 
         
     \end{array}{}
\end{equation}{}

The Nash value function $v^{Nash,j}(s)$ of an agent $j$ is defined as the expected cumulative discounted future rewards obtained by the agent $j$, given that all agents follow the Nash policy $\boldsymbol{\pi}_*$.

Here, we set, $F_t(s_t,\boldsymbol{a}_t) = F_t^{Q,j}(s_t,\boldsymbol{a}_t)$ = $C^j_t(s_t,\boldsymbol{a}_t) = 0 $ if $(s,\boldsymbol{a}) \neq (s_t, \boldsymbol{a}_t) $. Let the $\sigma$-field generated by all the random variables provided by $ (s_t, \alpha_t, a^1_t, \ldots, a^n_t, r_{t-1}, \ldots, s_1, \alpha_1, a_1, Q_0 )$, be represented by $\mathscr{P}_t$. Now, all the $Q$-values are $\mathscr{P}_t$ measurable which makes $\Delta_t$ and $F_t$, $\mathscr{P}_t$ measurable and this satisfies the measurability condition of Lemma~\ref{lemma:randomprocess}.

Hu and Wellman~\cite{hu2003nash} proved that the result, $v^{Nash,j}(s_{t+1}) \triangleq v^{j}(s', \pi^1_*, \ldots, \pi^n_*) = \pi^1_*(s') \cdots \pi^n_{*}(s') Q^j_*(s')$ holds (see the proof in Lemma 10 of \cite{hu2003nash}). Hence, from Lemma \ref{lemm:nashoperator}, we can show that the $\E[F^{Q}_t]$ forms a contraction mapping. This can be done using the fact that $\E(P_t Q_*) = Q_*$ (refer to Lemma 11 in  \cite{hu2003nash}). Here, the norm is the maximum norm on the joint action. 

Now, we have the following for all $t$,

\begin{equation}
    \begin{array}{l}
         ||\E[F_t^{Q,j}(s_t,\boldsymbol{a}_t) | \mathscr{P}_t] ||  \leq \gamma||Q^{j}_t(s_t,\boldsymbol{a}_t) 
         - Q^j_*(s_t,\boldsymbol{a_t}) || = \gamma||\Delta_t||
    \end{array}{}
\end{equation}

\noindent Now from Eq. \ref{eq:contraction},
\begin{equation}
    \begin{array}{l}
         ||\E[F_t(s_t,\boldsymbol{a}_t) | \mathscr{P}_t] || \\ \\ \leq
         ||\E[ F_t^{Q,j}(s_t,\boldsymbol{a}_t) | \mathscr{P}_t] ||
        +  ||\E [C^j_t (s_t,\boldsymbol{a}_t) | \mathscr{P}_t] ||
         \\ \\
        \leq \gamma ||\Delta_t || + || \E[C^j_t(s_t,\boldsymbol{a}_t)|\mathscr{P}_t]||
    \end{array}
\end{equation}

\noindent This satisfies the third condition of  Lemma~\ref{lemma:randomprocess}, provided that $c_t = || \E[C^j_t(s_t,\boldsymbol{a}_t)|\mathscr{P}_t]|| $ converges to 0 with probability 1 (w. p. 1.).  From the definition of $C^j_t$ and the Assumption 3, it can be shown that the value of $C^j_t$ converges to 0 in the limit of time (see Theorem~3 in Subramanian et al.~\cite{Subramanian2022multiagent}). \\

Thus, it follows from Lemma~\ref{lemma:randomprocess} that the process $\Delta_t$ converges to 0 and hence, low-$Q$ value for an agent $j$, converges to Nash $Q$ value $Q^j_*$.

\end{proof}

\section{Proof of Theorem~\ref{theorem:lowqsamplecomplexity}}\label{sec:polynomialproof}

In this section, we give the proof for Theorem~\ref{theorem:lowqsamplecomplexity}. For providing this bound we use the notion of covering time $L$. The covering time means that within $L$ steps from any start state, all state-joint action pairs are performed at least once by all agents. Similar to Dar and Mansour~\cite{Dar2003learning}, we clarify that we do not need to assume that the state-joint action pairs are being generated by any particular strategy. Same as in Appendix~\ref{sec:convergence}, across this section, we will use $Q$ to denote the low-$Q$ values. For a representative agent $j$, we will focus on the value of $r^j_t = ||Q^j_t - Q^j_*||$ and the aim is to bound the time until $r^j_t \leq \epsilon$. Here the norm denotes the maximum difference of the $Q$ values across all states and joint actions. The proof of this theorem follows the Theorem~4 in Dar and Mansour~\cite{Dar2003learning}. While the work of Dar and Mansour was restricted to single-agent MDPs, our result extends the analysis of Dar and Mansour to the general-sum stochastic game setting.

In line with the Eq.~\ref{eq:lowlevelq}, let us consider a stochastic iterative process of the form,  

\begin{equation}\label{eq:stochastictechnique}
    X^j_{t+1}(i) = (1 - \alpha_t(i))X^j_t(i) + \alpha_t(i)((H_t X^j_t)(i) + w^j_t(i)).
\end{equation}


As mentioned in Section~\ref{sec:theoriticalresults}, let us specify $\beta = \frac{1 - \gamma}{2}$. Also, consider a constant $Q^j_{\max}$ where the $Q^j_{\max}$ denotes the maximum low-$Q$ value possible to be obtained in the stochastic game by the agent $j$. Hence, the relation $||X_0|| \leq Q^j_{\max}$ holds. Further, let us consider a sequence $D^j_k$, with $D^j_1 = Q^j_{\max}$  and $D^j_{k+1} = (1-\beta) D^j_k$ for all $k \geq 1$. By the nature of this construction, the sequence $D^j_k$ is guaranteed to converge to 0, since at each step the value of $D^j_k$ is being continuously multiplied by a fractional value. Now we can prove the following result. 

\begin{lemm}\label{lemm:boundonX_t}
For every $k$, there exists a time $\tau_k$ such that, for any $t \geq \tau_k$ we have $||X^j_t|| \leq D^j_k$. 
\end{lemm}

\begin{proof}
See Theorem~8 in Dar and Mansour~\cite{Dar2003learning}. 
\end{proof}

The Lemma~\ref{lemm:boundonX_t} guarantees that at time $t \geq \tau_k$, for any $i$ the value of $||X^j_t(i)||$ is in the interval $[-D^j_k, D^j_k]$.

We can state the following lemma, where we bound the number of iterations until $D^j_i \leq \epsilon$.

\begin{lemm}\label{lemm:Dkbound}
For $m \geq \frac{1}{\beta} \ln(Q^j_{\max}/\epsilon)$ we have $D^j_m \leq \epsilon$. 

\end{lemm}

\begin{proof}
Since we have that $D^j_1 = Q^j_{\max}$ and $D^j_i = (1-\beta)D^j_{i-1}$, we will need a $m$ that satisfies $D^j_m = Q^j_{\max}(1-\beta)^m \leq \epsilon$. By taking a logarithm on both sides, we get 
\begin{equation}
    \begin{array}{l}
      \ln \Big( (Q^j_{\max})(1-\beta)^m \Big) \leq \ln (\epsilon) 
    \\ \\
    \ln (Q^j_{\max}) + m \ln (1-\beta) \leq \ln (\epsilon)
      \\  \\ 
      \implies 
      \ln (Q^j_{\max}) - \ln (\epsilon) \leq - m \ln (1-\beta) 
      \\ \\ 
      \implies
     \ln(Q^j_{\max}/\epsilon) \leq m \sum_{k=0}^\infty \beta^k/k
     \\ \\ 
     \implies
     1/\beta\ln(Q^j_{\max}/\epsilon) \leq m. 
    \end{array}
\end{equation}

\noindent In the last step we omit the higher powers since $\beta$ is a small fraction. This proves the result. 
\end{proof}

Now, we define a sequence of times $\tau_k$, with reference to the MA-TLQL low-level $Q$-updates with a polynomial learning rate. Let us define $\tau_{k+1} = \tau_k + L \tau^\omega_k$. Here the term $\omega$ denotes the decay factor of the learning rate with $\alpha^\omega_t = \frac{1}{(t+1)^\omega}$ for $\omega \in (1/2, 1)$. The term $L \tau^\omega_k$ specifies the number of steps needed to update each state joint action pair at least $\tau^\omega_k$ times. The time between the $\tau_k$ and $\tau_{k+1}$ is denoted as the $k^{\textrm{th}}$ iteration. Now we provide a definition for the number of times a state joint action pair is visited. 

\begin{defn}
Let $n(s, \boldsymbol{a},t_1, t_2)$ be the number of times that the state joint action pair $(s, \boldsymbol{a})$ was performed in the time step interval $[t_1, t_2]$. 
\end{defn}

Before providing the suitable bounds, we would like to provide some equations that relate the stochastic iterative technique given in Eq.~\ref{eq:stochastictechnique} with the $Q$-update in Eq~\ref{eq:lowlevelq}.

First we provide a formal definition of a stochastic game, that will be useful for our further analysis. Consider a stochastic game that can be defined as follows: 

\begin{defn}\label{defn:stochasticgame}
A stochastic game is defined as $\langle \mathcal{S},N,\mathbf{A},P,\mathbf{R}, \gamma \rangle $ where
$\mathcal{S}$ is a finite set of states, $N$ is the finite set of agents, $|N|=n$, and $\mathbf{A}=A^1\times\ldots \times A^n$ is the set of joint actions, where $A^j$ is the finite action set of an agent $j$, and  $\boldsymbol{a}=(a^1,\ldots,a^n)\in \mathbf{A}$ is the joint action where an agent $j$ takes action $a^j\in A^j$. Furthermore, $P_{i,k}(\boldsymbol{a}):S\times\mathbf{A}\times S\mapsto [0,1]$ is the transition function that provides the probability of reaching state $k$ from state $i$ when all agents are performing the joint action $\boldsymbol{a} \in \mathbf{A}$ in state $i$, $\boldsymbol{R}(s, \boldsymbol{a}) = \{R^1(s, \boldsymbol{a}), \ldots, R^n(s, \boldsymbol{a})\}$ is the set of reward functions, where  $R^j (s, \boldsymbol{a}):S\times\mathbf{A}\mapsto\mathbb{R}^n$ is the reward function of the agent $j$, and $\gamma$ is the discount factor satisfying $0\leq\gamma<1$.
\end{defn}

Towards the same, we define an operator $H$ that can be represented as 

\begin{equation}
    \begin{array}{l}
         (HQ^j)(i,\boldsymbol{a}) =  \sum_{k=0}^{|S|} P_{ik}(\boldsymbol{a}) (R^j(i,\boldsymbol{a}) + \gamma \max_{b^j \in A^j} Q^j(k, \boldsymbol{b})).
    \end{array}
\end{equation}

\noindent Here the $\boldsymbol{b} = \{b^j, b^{-j}\} $, where $b^{-j}$ denotes the joint action of all agents except the agent $j$.

Rewriting the $Q$-function with $H$ we get, 

\begin{equation}
\begin{array}{l}
    Q^j_{t+1}(i,\boldsymbol{a}) = (1-\alpha_t(i,\boldsymbol{a})) Q^j_t(i,\boldsymbol{a}) +  \alpha_t(i,\boldsymbol{a})((HQ^j_t)(i,\boldsymbol{a}) + w^j_t(i,\boldsymbol{a})). 
    \end{array}
\end{equation}

Let $\overline{i}$ be the state that is reached by performing joint action $\boldsymbol{a}$ at time $t$ in state $i$ and $r^j(i,\boldsymbol{a})$ be the reward observed by the agent $j$ at state $i$; then 

\begin{equation}
    \begin{array}{l}
         w^j_t(i,\boldsymbol{a}) = 
         r^j(i,\boldsymbol{a}) + \gamma \max_{b^j \in A^j} Q^j_t(\overline{i}, \boldsymbol{b})  \\ \quad \quad \quad \quad \quad - \sum_{k=0}^{|S|} P^j_{ik}(\boldsymbol{a}) \Big( R^j(i,\boldsymbol{a}) + \gamma \max_{b^j \in A^j} Q^j_t(k,\boldsymbol{b}) \Big). 
    \end{array}
\end{equation}

From this construction $w^j_t$ is bounded by $Q^j_{\max}$ for all $t$ and has zero expectation. Further we will define two other sequences $W^j_{t;\tau}$ and $Y^j_{t; \tau}$ where $\tau$ represents some initial time. These are given by the following equations. 

\begin{equation}
    W^j_{t+1; \tau_k}(s, \boldsymbol{a}) = (1 - \alpha^\omega_t(i)) W^j_{t;\tau_k}(s, \boldsymbol{a}) + \alpha^\omega_t(i)w^j_t(s, \boldsymbol{a})
\end{equation}

where $W^j_{\tau_k;\tau_k}(s, \boldsymbol{a}) = 0$. The value of $W^j_{t;\tau_k}$ bounds the contributions of $w^j_t(s, \boldsymbol{a})$, to the value of $Q^j_t$, starting from an arbitrary $\tau_k$. Now we have 

\begin{equation}\label{eq:defnY}
    Y^j_{t+1; \tau_k}(s, \boldsymbol{a}) = (1 - \alpha^\omega_t(i))Y^j_{t;\tau_k}(s, \boldsymbol{a}) + \alpha^\omega_t(s, \boldsymbol{a}) \gamma D^j_k
\end{equation}

\noindent  where $Y^j_{\tau_k; \tau_k} = D^j_k$. 

Next, we can state a lemma that will bound the $Q$-functions w.r.t the sequences $W^j_{t;\tau_k}$ and $Y^j_{t; \tau_k}$.

\begin{lemm}\label{lemma:boundlemma}
For every state $s$ and joint action $\boldsymbol{a}$ and time $\tau_k$, we have 
\begin{equation}
    \begin{array}{l}
         -Y^j_{t;\tau_k}(s, \boldsymbol{a}) + W^j_{t;\tau_k}(s, \boldsymbol{a}) \\ \\ \leq Q^j_*(s, \boldsymbol{a}) - Q^j_t(s, \boldsymbol{a}) \\ \\ \leq Y^j_{t;\tau_k}(s, \boldsymbol{a}) + W^j_{t;\tau_k}(s, \boldsymbol{a})
    \end{array}
\end{equation}
\end{lemm}
\begin{proof}
See Lemma~4.4 in Bertsekas and Tsitsiklis~\cite{berksekas1996neuro}. 
\end{proof}

From the Lemma~\ref{lemma:boundlemma} we see that a bound on the difference $r^j_t$ depends on the bound for $Y^j_{t;\tau_k}$ and $W^j_{t;\tau_k}$. So we can bound $Y^j_{t;\tau_k}$ and $W^j_{t;\tau_k}$ separately and two bounds together will provide a bound for $r^j_t$. 

We first provide a result on the nature of the sequence $Y^j_{\tau_k; \tau_k}$. 

\begin{lemm}\label{lemma:Yknature}
The sequence $Y^j_{t; \tau_k}$ is a monotonically decreasing sequence.  
\end{lemm}
\begin{proof}
From Eq.~\ref{eq:defnY} we can write (subtract $\gamma D^j_k$ from both sides), 

\begin{equation}
\begin{array}{l}
    (Y^j_{t+1; \tau_k}(s, \boldsymbol{a}) - \gamma D^j_k) \\ \\  = (1 - \alpha^\omega_t(i))Y^j_{t;\tau_k}(s, \boldsymbol{a}) + (1 - \alpha^\omega_t(s, \boldsymbol{a})) \gamma D^j_k
    \\ \\ 
     = (1 - \alpha^\omega_t(i)) (Y^j_{t;\tau_k}(s, \boldsymbol{a}) + \gamma D^j_k).
    \end{array}
\end{equation}

Now, the convergence of $|| Y^j_{t+1; \tau_k}(s, \boldsymbol{a}) - \gamma D^j_k||$ follows since $\lim_{n \xrightarrow{} \infty} \Pi^n_{t = k}(1 - \alpha^\omega_t(s, \boldsymbol{a})) = 0$. This shows that the sequence $(Y^j_{t+1; \tau_k}(s, \boldsymbol{a}) - \gamma D^j_k)$ monotonically decreases to 0 and hence the sequence $Y^j_{t; \tau_k}$ monotonically decreases to $\gamma D_k$. This  proves our result.

\end{proof}

Next, we provide a bound on the value of $Y^j_{t; \tau_k}$. 

\begin{lemm}\label{lemma:Yboundlemma}
Consider the low-$Q$ update given in Eq.~\ref{eq:lowlevelq}, with a polynomial learning rate and assume that for any $t \geq \tau_k$ we have $Y^j_{t;\tau_k}(s,\boldsymbol{a}) \leq D_k$. Then for any $t \geq \tau_k + L \tau^\omega_k = \tau_{k+1}$ we have $Y^j_{t;\tau_k}(s,\boldsymbol{a}) \leq D^j(\gamma + \frac{2}{e}\beta)$. 
\end{lemm}

\begin{proof}

For each state-joint action pair $s, \boldsymbol{a}$ we are assured that $n(s,a, \tau_k, \tau_{k+1}) \geq \tau_k^\omega$, since the covering time is $L$ and the underlying policy has made $L \tau^\omega_k$ steps (since we have the relation $\tau_{k+1} = \tau_k + L\tau^\omega_k$).

Let $Y^j_{\tau_k, \tau_k}(s, \boldsymbol{a}) = \gamma D^j_k + \rho^j_{\tau_k}$, where $\rho_{\tau_k} = (1-\gamma) D^j_k$. Now we have the following expression, 

\begin{equation}
    \begin{array}{l}
         Y^j_{t+1, \tau_k}(s, \boldsymbol{a}) = (1-\alpha^\omega_t) Y^j_{t;\tau_k}(s, \boldsymbol{a}) + \alpha^\omega_t \gamma D^j_k 
         = \gamma D^j_k + (1 - \alpha_t^\omega) \rho^j_t
    \end{array}
\end{equation}

\noindent where $\rho^j_{t+1} = \rho^j_t (1 - \alpha_t^\omega)$. We aim to show that after time $\tau_{k+1} = \tau_k + \tau_k^\omega$, for any $t \geq \tau_k^\omega$ for any $t \geq \tau_{k+1}$ we have $\rho^j_t \leq \frac{2}{e} \beta D^j_k$. By definition, we can rewrite $\rho^j_t$ as,

\begin{equation}\label{eq:problem1}
    \begin{array}{l}
         \rho^j_t = (1 - \gamma) D^j_k \Pi^{t-\tau_k}_{l=1} (1 - \alpha^\omega_{l+\tau_k}) \\ \\   = 2 \beta D^j_k \Pi_{l=1}^{t-\tau_k} (1 - \alpha^\omega_{l+\tau_k}) \\ \\ = 2 \beta D^j_k \Pi_{l=1}^{t-\tau_k} (1 - \frac{1}{|n(s,\boldsymbol{a}, \tau_{k+1}, l)|^\omega})
    \end{array}
\end{equation}

\noindent the last identity follows from the definition of $\alpha_t^\omega$.

Since the $\tau^\omega_k$'s are monotonically decreasing,
\begin{equation}
    \rho^j_t \leq 2 \beta D^j_k (1 - \frac{1}{\tau^\omega_k})^{t - \tau_k}. 
\end{equation}

For $t \geq \tau_k + \tau^\omega_k$ we have, 

\begin{equation}\label{eq:problem2}
    \rho^j_t \leq 2 \beta D^j_k \Big( 1 - \frac{1}{\tau^\omega_k} \Big)^{\tau^\omega_k} \leq \frac{2}{e} \beta D^j_k. 
\end{equation}

\noindent The last step is obtained from the fact that $\lim_{x \xrightarrow{} \infty} (1 - (1/x))^x = 1/e$.

Hence, $Y_{t;\tau_k}(s, \boldsymbol{a}) \leq (\gamma + \frac{2}{e} \beta) D_k$.

\end{proof}

From Lemma~\ref{lemma:Yboundlemma} we have provided a bound for the term $Y^j_{t;\tau_k}$ for $t = \tau_{k+1}$ which will automatically hold for all $t \geq \tau_{k+1}$, since the term $Y^j_{t;\tau_k}$ is monotonically decreasing (from Lemma~\ref{lemma:Yknature}) and deterministic (from Eq.~\ref{eq:defnY} and Lemma~\ref{lemma:boundlemma}).

Next we bound the term $W^j_{t;\tau_k}$ by $(1-\frac{2}{e})\beta D^j_k$. The sum of the bounds for $W^j_{t;\tau_k}(s, \boldsymbol{a})$ and $Y^j_{t;\tau_k}(s, \boldsymbol{a})$, would be $(\gamma + \beta) D^j_k = (1 - \beta) D^j_k = D^j_{k+1}$, as desired. 

Now we state a definition for a sequence $\eta_i^{k,t,j}(s, \boldsymbol{a})$ and $W^{l,j}_{t;\tau_k}$. 

\begin{defn}
Let 
\begin{equation}
\begin{array}{l}
W^j_{t;\tau_k}(s, \boldsymbol{a})  = (1 - \alpha^\omega_t(s, \boldsymbol{a})) W^j_{t-1; \tau_k}(s, \boldsymbol{a})  + \alpha^\omega_t(s, \boldsymbol{a}) w^j_t(s, \boldsymbol{a})  \\ \\ = \sum_{i = 1}^{t-\tau_k} \eta^{k,t,j}_{i+\tau_k} (s, \boldsymbol{a}) w^j_{i+ \tau_k} (s, \boldsymbol{a}), \\ \\  
\textrm{ where } \\ \\  \eta^{k,t,j}_{i+\tau_k}(s,\boldsymbol{a}) = \alpha^\omega_{i + \tau_k}(s,\boldsymbol{a}) \Pi^t_{l = \tau_k + i + 1} (1 - \alpha^\omega_l(s,\boldsymbol{a})).
\end{array}
\end{equation}

\end{defn}

For bounding the sequence $W^j_{t; \tau_k}$ we consider the interval $t \in [\tau_{k+1}, \tau_{k+2}]$. First we provide a lemma that bounds the coefficients in this interval and bounds the influence of $w^j_t(s, \boldsymbol{a})$ in this interval.

\begin{lemm}\label{lemma:wbound}
Let $\Tilde{w}^{j,t}_{i + \tau_k}(s,\boldsymbol{a}) = \eta_{i+\tau_k}^{k,t,j}(s,\boldsymbol{a})w^j_{i + \tau_k}(s,\boldsymbol{a})$, then for any $t \in [\tau_{k+1}, \tau_{k+2}]$ the random variable $\Tilde{w}^j_{i + \tau_k} (s,\boldsymbol{a})$ has zero mean and bounded by $(L/\tau_k)^\omega Q^j_{\max}$.
\end{lemm}

\begin{proof}

Since $\eta_{i+\tau_k}^{k,t,j}((s,\boldsymbol{a})) = \alpha^\omega_{i + \tau_k}(s,\boldsymbol{a}) \Pi^t_{l = \tau_k + i + 1} (1 - \alpha^\omega_l (s,\boldsymbol{a}))$, we can divide $\eta^{k,t,j}_{i+\tau_k}$ into two parts, the first $\alpha^\omega_{i + \tau_k}$ and the second $\mu = \Pi^t_{l = \tau_k + i + 1}(1 - \alpha^\omega_l)$. Since, $\mu$ is bounded from above by 1, we have, 

\begin{equation}
    \begin{array}{l}
         \eta_{i+\tau_k}^{k,t,j}(s,\boldsymbol{a}) \leq \alpha^\omega_{i + \tau_k} (s,\boldsymbol{a})
         \\ \\ 
         = \frac{1}{|\#(s, \boldsymbol{a},i+\tau_k)|} \overset{*}{\leq} (\frac{L}{i+\tau_k})^\omega \leq (\frac{L}{\tau_k})^\omega. 
         
    \end{array}
\end{equation}

\noindent Here, the $*$ is from the fact that in a time interval of $\tau$, each state-joint action pair is performed at-least $\tau/L$ times by the definition of covering time. 




Hence, we get the relation that $\eta^{k,t,j}_{i+\tau_k}(s,\boldsymbol{a}) \leq (L/\tau_k)^\omega$.

Now, consider the expectation of $\Tilde{w}_{i + \tau_k}(s,\boldsymbol{a})$. By definition, we have that, $w^j_{\tau_k + i}(s,\boldsymbol{a})$ has zero mean and is bounded by $Q^j_{\max}$ for any history and state-joint action pair, hence, 

\begin{equation}
    \begin{array}{l}
         \E[\Tilde{w}_{i + \tau_k}(s,\boldsymbol{a})] 
         \\ \\
         = \E[\eta_{i+\tau_k}^{k,t,j}(s,\boldsymbol{a})w_{i+\tau_k}(s,\boldsymbol{a})] 
         \\ \\
         = \eta_{i+\tau_k}^{k,t,j}(s,\boldsymbol{a}) \E[w_{i+\tau_k}(s,\boldsymbol{a})] = 0.
    \end{array}
\end{equation}

Next we can prove that it is bounded as well,

\begin{equation}
\begin{array}{l}
     |\Tilde{w}_{i + \tau_k}(s,\boldsymbol{a})| 
     \\ \\ 
     = | \eta^{k,t,j}_i(s,\boldsymbol{a})w_{i+\tau_k}(s,\boldsymbol{a})|
     \\ \\ 
     \leq |\eta^{k,t,j}_i(s,\boldsymbol{a})Q^j_{\max} |
     \\ \\ 
     \leq (L/\tau_k)^\omega Q^j_{\max}
\end{array}
\end{equation}

\end{proof}

Now, let us define $W_{t;\tau_k}^{l,j}(s, \boldsymbol{a}) = \sum_{i=1}^l \Tilde{w}^t_{i + \tau_k} (s, \boldsymbol{a})$. The objective is to prove that this is a martingale difference sequence having bounded differences.

\begin{lemm}\label{lemma:martingale}
For any $t \in [\tau_{k+1}, \tau_{k+2}]$ and $1 \leq l \leq t$ we have that $W^{l,j}_{t;\tau_k}(s,\boldsymbol{a})$ is a martingale sequence that satisfies,
\begin{equation}
    |W^{l,j}_{t;\tau_k}(s,\boldsymbol{a}) - W^{l-1,j}_{t;\tau_k}(s,\boldsymbol{a}) | \leq (L / \tau_k)^\omega Q^j_{\max}
\end{equation}
\end{lemm}

\begin{proof}

We first note that the term $W^{l,j}(s, \boldsymbol{a})$ is a martingale sequence, since 

\begin{equation}
    \begin{array}{l}
         \E[W^{l,j}_{t;\tau_k}(s, \boldsymbol{a}) - W^{l-1,j}_{t;\tau_k}(s, \boldsymbol{a})| F_{\tau_k + l - 1}] \\ \\ 
         = \E [\Tilde{w}^{j,t}_{l + \tau_k}(s, \boldsymbol{a})| F_{\tau_k + l - 1}] = 0
    \end{array}
\end{equation}

\noindent where the variable $F_{\tau_k + l - 1}$ denotes all previous values of $W^l$. 

Also, by Lemma~\ref{lemma:wbound} we can show that $\Tilde{w}^j_{l+\tau_k}(s, \boldsymbol{a})$ is bounded by $(L/\tau_k)^\omega Q^j_{\max}$, thus

\begin{equation}
  \begin{array}{l}
     |W^{l,j}_{t;\tau_k}(s, \boldsymbol{a}) - W^{l-1, j}_{t;\tau_k}(s, \boldsymbol{a})| 
     \\ \\
     = \Tilde{w}^j_{l+\tau_k}(s, \boldsymbol{a}) \leq (L/\tau_k)^\omega Q^j_{\max} 
\end{array}  
\end{equation}

\end{proof}

The next lemma bounds the term $W^j_{t;\tau_k}$.

\begin{lemm}\label{lemma:probabilitybound}
Consider the low-$Q$ update given in Eq.~\ref{eq:lowlevelq}, with a polynomial learning rate. With probability at least $1 - \frac{\delta}{m}$ we have that, for every state-joint action pair $|W^j_{t;\tau_k}(s, \boldsymbol{a})| \leq (1 - \frac{2}{e} \gamma D_k)$ for any $t \in [\tau_{k+1}, \tau_{k+2}]$, i.e. 

\begin{equation}
    \begin{array}{l}
         Pr\Big[ \forall s, \boldsymbol{a} \forall t \in [\tau_{k+1}, \tau_{k+2}]: |W_{t;\tau_k}(s, \boldsymbol{a})| \leq (1 - \frac{2}{e}) \beta D_k \Big] 
         \geq 1 - \frac{\delta}{m}
    \end{array}
\end{equation}

given that 
\begin{equation}
    \begin{array}{l}
         \tau_k = \Theta \Big(( \frac{L^{1 + 3\omega}Q^{2,j}_{\max}  \ln(Q^j_{\max}|S|\Pi_i |A|_i m/(\delta \beta D_k))}{\beta^2 D_k^2})^{1/\omega}\Big)
    \end{array}
\end{equation}

\end{lemm}

\begin{proof}

For each state-joint action pair comparing $W^{l,j}_{t;\tau_k}(s, \boldsymbol{a})$ and $W^{j}_{t;\tau_k}(s, \boldsymbol{a})$ we note that $W^{j}_{t;\tau_k}(s, \boldsymbol{a}) = W_{t;\tau_k}^{t-\tau_k + 1,j}(s, \boldsymbol{a})$.

Let $l = n(s, \boldsymbol{a}, \tau_k, t)$, then for any $t \in [\tau_{k+1}, \tau_{k+2}]$ we have that 

\begin{equation}
    \begin{array}{l}
        l \leq \tau_{k+2} - \tau_k = \tau_{k+1} + L \tau^\omega_{k+1} - \tau_k \\ \\ = \tau_{k} + L \tau^\omega_{k}  + L (\tau_{k} + L \tau^\omega_{k}) - \tau_k  \\ \\ 
         = L \tau^\omega_{k}  + L\tau_{k} + L^2 \tau^\omega_{k}
        \leq \Theta(L^{1 + \omega} \tau^\omega_k).
    \end{array}
\end{equation}
By Lemma~\ref{lemma:martingale} we can apply Azuma's inequality to $W_{t;\tau_k}^{t-\tau_k+1,j}(s, \boldsymbol{a})$ with $c_i = (L/\tau_k)^\omega Q^j_{\max}$. Therefore, we can derive that 

\begin{equation}
    \begin{array}{l}
         Pr\Big[ |W_{t;\tau_k} (s, \boldsymbol{a}) | \geq \Tilde{\epsilon} |t \in [\tau_{k+1}, \tau_{k+2}] \Big]
         \\ \\ 
         \leq 2e^{\frac{-\Tilde{\epsilon}^2}{2 \sum^t_{i = \tau_k +1} c^2_i}}
         \leq 2e^{\frac{-\Tilde{\epsilon}^2}{2 l c^2_i}}
         \\ \\ 
         \leq 2e^{-c \frac{\Tilde{\epsilon}^2\tau^{2\omega}_k}{l L^{2\omega}Q_{\max}^{2,j} }}
         \\ \\ 
        \leq 2e^{-c \frac{\Tilde{\epsilon}^2\tau^{\omega}_k}{ L^{1 + 3\omega}Q_{\max}^{2,j} }}
    \end{array}
\end{equation}

\noindent  for some constant $c > 0$. We can set $\Tilde{\delta}_k = 2e^{\frac{-c\tau_k^\omega \Tilde{\epsilon}^2}{L^{1 + 3\omega} Q_{\max}^{2,j}}}$, which holds for $\tau^\omega_k = \Theta(\ln(1/\Tilde{\delta}_k) L^{1 + 3 \omega} Q_{\max}^{2,j}/\Tilde{\epsilon}^2)$. 

Using the union bound we get, 

\begin{equation}
    \begin{array}{l}
         Pr[\forall s, \forall \boldsymbol{a}, \forall t \in [\tau_{k+1}, \tau_{k+2}]: W^j_{t, \tau_k} (s, \boldsymbol{a}) \leq \Tilde{\epsilon}  ] 
         \\ \\ 
         \geq 1 - \sum_{t=\tau_{k+1}}^{\tau_{k+2}} Pr[\forall s, \forall \boldsymbol{a}, W^j_{t;\tau_k}(s, \boldsymbol{a}) \geq \Tilde{\epsilon}]
        \\ \\ 
       \geq 1 - \sum_{t=\tau_{k+1}}^{\tau_{k+2}} Pr[\forall s, \forall \boldsymbol{a}, |W^j_{t;\tau_k}(s, \boldsymbol{a})| \geq \Tilde{\epsilon}]
      \\ \\ 
         \geq  1 - \sum_{t=\tau_{k+1}}^{\tau_{k+2}} \Tilde{\delta}_k |S|\Pi_i |A|_i
        \\ \\ 
        \geq  1 - (\tau_{k+2} - \tau_{k+1}) \Tilde{\delta}_k |S|\Pi_i |A|_i
        
    \end{array}
\end{equation}

We would like to set a level of probability of $1 - \frac{\delta}{m}$, for each state-joint action pair. From the above equation we get,

\begin{equation}
    \begin{array}{l}
       1 - (\tau_{k+2} - \tau_{k+1}) \Tilde{\delta}_k |S| \Pi_i |A|_i = 1 - \delta/m
       \\ \\ 
       \implies
       (\tau_{k+2} - \tau_{k+1}) \Tilde{\delta}_k |S| \Pi_i |A|_i = \delta/m.
        \\ \\ 
       \implies
        \Tilde{\delta}_k = \delta/(\tau_{k+2} - \tau_{k+1})m |S| \Pi_i |A|_i.
    \end{array}
\end{equation}

Thus taking $\Tilde{\delta}_k = \frac{\delta}{m(\tau_{k+2} - \tau_{k+1})|S| |A| }$ assures $1 - \frac{\delta}{m}$ for each state-joint action pair. As a result we have, 

\begin{equation}
    \begin{array}{l}
         \tau^\omega_k 
         = \Theta \Big(\frac{L^{1+3\omega}Q^{2,j}_{\max} \ln(|S|\Pi_i |A|_i m \tau^\omega_k/\delta ) }{\Tilde{\epsilon}^2} \Big) 
         \\ \\
         = \Theta \Big(\frac{L^{1+3\omega}Q^{2,j}_{\max} \ln(|S|\Pi_i |A|_i m Q^{j}_{\max}/(\delta \Tilde{\epsilon}) ) }{\Tilde{\epsilon}^2} \Big)
    \end{array}
\end{equation}

Setting $\Tilde{\epsilon} = (1 - 2/e) \beta D_k$ gives the desired bound.

\end{proof}

Now that we have bounded for each iteration the time needed to achieve the desired probability $1 - \frac{\delta}{m}$. The next lemma provides a bound for error in all iterations.

\begin{lemm}\label{lemm:Wboundlemma}
Consider the low-$Q$ update given in Eq.~\ref{eq:lowlevelq}, with a polynomial learning rate. With probability $1-\delta$, for every iteration $k \in [1, m]$ and time $t \in [\tau_{k+1}, \tau_{k+2}] $ we have $|W^j_{t;\tau_k}(s, \boldsymbol{a}) | \leq (1 - \frac{2}{e})\beta D^j_k$, i.e., 
\begin{equation}
    \begin{array}{l}
         Pr \Big[ \forall k \in [1,m], \forall t \in [\tau_{k+1}, \tau_{k+2}], \forall s, \boldsymbol{a}:
         |W^j_{t;\tau_k} (s, \boldsymbol{a})| \\ \quad \quad \quad \quad \quad \leq (1 - \frac{2}{e}) \beta D^j_k \Big] 
         \geq 1 - \delta
    \end{array}
\end{equation}

given that 

\begin{equation}
  \tau_0 = \Theta \Big( \big( \frac{L^{1 + 3\omega} Q^{2,j}_{\max} \ln(Q^j_{\max} |S|\Pi_i |A|_im/(\delta \beta \epsilon))}{\beta^2 \epsilon^2} \big)^{1/\omega} \Big)  
\end{equation}

\end{lemm}

\begin{proof}

From Lemma~\ref{lemma:probabilitybound} we have that 

\begin{equation}
    \begin{array}{l}
         Pr\Big[ \forall t \in [\tau_{k+1}, \tau_{k+2}]: |W^j_{t;\tau_k}| \geq (1 - \frac{2}{e}) \beta D^j_k \Big] \leq \frac{\delta}{m}
    \end{array}
\end{equation}

Using the union bound we have that 

\begin{equation}
    \begin{array}{l}
         Pr[ \forall k \leq m, \forall t \in [\tau_{k+1}, \tau_{k+2}] |W^j_{t;\tau_k}| \geq \Tilde{\epsilon}]
         \\ \\ 
         \leq \sum_{k=1}^m Pr [\forall t \in [\tau_{k+1}, \tau_{k+2}] |W^j_{t;\tau_k}| \geq \Tilde{\epsilon}] \leq \delta
    \end{array}
\end{equation}

where $\Tilde{\epsilon} = (1 - \frac{2}{e}) \beta D_k$.

\end{proof}

The next lemma solves the recurrence $\sum_{i=0}^{m-1} L\tau^\omega_i + \tau_0$ and derives the time complexity. 

\begin{lemm}\label{lemm:abound}
Let 
\begin{equation}
    a_{k+1} = a_k + La^\omega_k = a_0 + \sum_{i=0}^k La^\omega_i
\end{equation}
Then for any constant $\omega \in (0,1)$, $a_k = \Omega(a_0^{1 - \omega}/L + L^{\frac{1}{1-\omega}} ((k+1)/2)^{\frac{1}{1-\omega}})$.

\end{lemm}

\begin{proof}

Let us define the following series 

\begin{equation}
    \begin{array}{l}
         b_{k+1} = \sum_{i=0}^k Lb^\omega_i + b_0
    \end{array}
\end{equation}

\noindent with an initial condition $b_0 = L^{\frac{1}{1-\omega}}$.






Now we lower bound $b_k$ by $(L(k+1)/2)^{\frac{1}{1-\omega}}$. We use induction to prove this hypothesis. For $k=0$, 

\begin{equation}
    b_0 = L^{\frac{1}{1-\omega}} \geq (\frac{L}{2})^\frac{1}{1-\omega}
\end{equation}

Assume that the induction hypothesis holds for $k-1$ and prove for $k$, 

\begin{equation}
    \begin{array}{l}
         b_k = b_{k-1} + L b^\omega_{k-1} 
         \\ \\ 
         = (Lk/2)^{\frac{1}{1-\omega}} + L(Lk/2)^{\frac{\omega}{1-\omega}}
         \\ \\ 
         = L^\frac{1}{1-\omega}((k/2)^\frac{1}{1-\omega} + (k/2)^\frac{\omega}{ 1 - \omega})
         \\ \\ 
         \geq L^\frac{1}{1-\omega}((k+1)/2)^\frac{1}{1-\omega}.
    \end{array}
\end{equation}

For $a_0 \geq L^\frac{1}{1-\omega}$ we can view the series as starting at $b_k = a_0$. Since the first time step is 1, we can see that the start point has moved $\Theta (a_0^{1-\omega}/L)$. Therefore, we have a total complexity of $\Omega(a_0^{1 - \omega}/L + L^{\frac{1}{1-\omega}} ((k+1)/2)^{\frac{1}{1-\omega}})$.

\end{proof}

\begin{theorem2}
Let us specify that with probability at least $1 - \delta$, for an agent $j$, $||Q^j_T - Q^j_* ||_\infty \leq \epsilon$. The bound on the rate of convergence of low-$Q$, $Q^j_T$, with a polynomial learning rate of factor $\omega$ is given by (with $Q^j_*$ as the Nash $Q$-value of $j$)

\begin{equation*}
\begin{array}{l}
      T = \Omega \Big( \Big(\frac{L^{1 + 3\omega}Q^{2,j}_{\max}\ln(\frac{|S|\Pi_i |A_i| Q^j_{\max}}{\delta \beta \epsilon})}{\beta^2 \epsilon^2} \Big)^{1 - \omega}/L \\ \quad \quad \quad \quad \quad \quad \quad \quad \quad \quad \quad + \Big( (\frac{L}{\beta} \ln \frac{Q^j_{\max}}{\epsilon}  + 1)/2 \Big)^{\frac{1}{1-\omega}}\Big).  
    \end{array}
\end{equation*}

\end{theorem2}

\begin{proof}
The proof of Theorem~\ref{theorem:lowqsamplecomplexity} follows from Lemmas~\ref{lemma:Yboundlemma}, \ref{lemm:Wboundlemma}, \ref{lemm:Dkbound}, and \ref{lemm:abound}. 

Specifically, from the relation in Lemma~\ref{lemm:abound} substitute the value of $a_0$ from the Lemma~\ref{lemm:Wboundlemma} (value of $\tau_0$), and value of $k$ from Lemma~\ref{lemm:Dkbound} (lower bound for $m$). From the Lemma~\ref{lemma:Yboundlemma} and Lemma~\ref{lemm:Wboundlemma}, we see that the condition required in Lemma~\ref{lemma:boundlemma} is satisfied to provide a lower bound. 

\end{proof}

\section{Proof of Theorem~\ref{theorem:linearlearningrate}}\label{sec:linaerlearningrate}

In this section, we aim to show that the size of the $k$th iteration is $L(1 + \psi) \tau_k$ for some positive constant $\psi \leq 0.712$. The covering time property guarantees that in $(1 + \psi) L \tau_k$ steps, each pair of state-joint actions are performed at least $(1 + \psi) \tau_k$ times. The sequence of times in this case is $\tau_{k+1} = \tau_k + (1 + \psi) L\tau_k$. As in the last section, we will first bound $Y^j_{t;\tau_k}$ and then bound $W^j_{t;\tau_k}$. As in Appendix~\ref{sec:convergence}, we will use $Q$ to denote the low-$Q$ values across this section as well. The proof of this theorem follows the Theorem~5 in Dar and Mansour~\cite{Dar2003learning}. While the work of Dar and Mansour was restricted to single-agent MDPs, our result extends the analysis of Dar and Mansour to the general-sum stochastic game setting. 

\begin{lemm}\label{lemma:Ytlinearbound}
Consider the low-$Q$ update given in Eq.~\ref{eq:lowlevelq}, with a linear learning rate. Assume that for $t \geq  \tau_k$ we have that $Y^j_{t;\tau_k} (s, \boldsymbol{a}) \leq D^j_k$. Then for any $t \geq \tau_k + (1 + \psi) L \tau_k = \tau_{k+1}$ we have that $Y^j_{t;\tau_k} (s, \boldsymbol{a}) \leq (\gamma + \frac{2}{2 + \psi}\beta )D^j_k$.
\end{lemm}

\begin{proof}

For each state-joint action pair, we are assured that $n(s, \boldsymbol{a}, \tau_k, \tau_{k+1}) \geq (1 + \psi) \tau_k$, since in an interval of $(1 + \psi)L\tau_k$ steps, each state-joint action pair is visited at least $(1 + \psi) \tau_k$ times by the definition of covering time.

Let $Y^j_{\tau_k, \tau_k} (s, \boldsymbol{a}) = \gamma D^j_k + \rho^j_{\tau_k}$, where $\rho^j_{\tau_k} = (1 - \gamma)D^j_k$. We now have, 

\begin{equation}
    \begin{array}{l}
         \rho^j_t = (1 - \gamma) \Pi_{l=1}^{t-\tau_k} (1 - \alpha_{l+\tau_k})
         \\ \\
         = 2 \beta D^j_k \Pi_{l=1}^{t - \tau_k} (1 - \alpha_{l+\tau_k}) 
         \\ \\ 
         = 2 \beta D^j_k \Pi_{l=1}^{t - \tau_k} (1 - \frac{1}{|n(s,\boldsymbol{a}, \tau_{k+1}, l)|}), 
    \end{array}
\end{equation}

\noindent where the last identity follows from the fact that $\alpha_t = \frac{1}{n(s,\boldsymbol{a}, 0, t)}$. Since the $\tau_k$'s are monotonically decreasing, using $t = \tau_k + (1+\psi)L\tau_k$, we get (using $\psi<0.712$), 

\begin{equation}
    \begin{array}{l}
         \rho_t \leq 2 D^j_k \beta (1 - \frac{1}{(1 + \psi)\tau_k})^{t - \tau_k}
         \\ \\
          \leq 
         2 D^j_k \beta (1 - \frac{1}{(1 + \psi)\tau_k})^{1 + \psi L \tau_k}
         \\ \\ 
         \leq 
         2 D^j_k \beta (1 - \frac{1}{(1 + \psi)\tau_k})^{1 + \psi \tau_k}
         \\ \\
         \leq 
         \frac{2 D^j_k \beta}{e} 
         \\ \\ 
         \leq  \frac{2 D^j_k \beta}{2 + \psi}
    \end{array}
\end{equation}

Hence, $Y^j_{t;\tau_k}(s,\boldsymbol{a}) \leq (\gamma + \frac{2}{2 + \psi} \beta) D^j_k$. 

\end{proof}

The following lemma enables the use of Azuma's inequality. 

\begin{lemm}\label{lemma:wmartingalelinear}
For any $t \geq \tau_k$ and $1 \leq l \leq t$ we have that $W^{l,j}_{t;\tau_k}(s, \boldsymbol{a})$ is a martingale sequence, which satisfies, 

\begin{equation}
    |W^{l,j}_{t;\tau_k}(s, \boldsymbol{a}) -  W_{t;\tau_k}^{l-1,j}(s, \boldsymbol{a})| \leq \frac{Q^j_{\max}}{n(s,\boldsymbol{a}, 0, t)}
\end{equation}
\end{lemm}

\begin{proof}

$W^{l,j}_{t;\tau_k}(s, \boldsymbol{a})$ is a martingale difference sequence since, 

\begin{equation}
    \begin{array}{l}
         \E[W^{l,j}_{t;\tau_k}(s, \boldsymbol{a}) - W^{l-1,j}_{t;\tau_k}(s, \boldsymbol{a}) | F_{\tau_k + l - 1}] 
         \\ \\
         = 
         \E[\eta^{k,t,j}_{\tau_k + l}(s, \boldsymbol{a})w^j_{\tau_k + l}(s, \boldsymbol{a}) | F_{\tau_k + l - 1}] 
        \\ \\ 
        = 
        \eta^{k,t,j}_{\tau_k + l}(s, \boldsymbol{a}) \E[w^j_{\tau_k + l}(s, \boldsymbol{a}) | F_{\tau_k + l - 1}] = 0

    \end{array}
\end{equation}

For linear learning rate we have that $\eta^{k,t,j}_{\tau_k + l}(s, \boldsymbol{a}) 
\leq \alpha_{l+\tau_k} = 1/n(s,\boldsymbol{a}, 0, t)$, thus 

\begin{equation}
    \begin{array}{l}
         |W^{l,j}_{t;\tau_k}(s, \boldsymbol{a}) - W^{l-1,j}_{t;\tau_k}(s, \boldsymbol{a}) | 
         \\ \\ 
         = \eta^{k,t,j}_{\tau_k + l}(s, \boldsymbol{a}) |w^j_{\tau_k + l}(s, \boldsymbol{a}) | \\ \\ \leq \frac{Q^j_{\max}}{n(s,\boldsymbol{a}, 0, t)}. 
    \end{array}
\end{equation}

\end{proof}

The following lemma provides a bound for the stochastic term $W^j_{t;\tau_k}$.

\begin{lemm}\label{lemma:Wtlinearbound}
Consider the low-$Q$ update given in Eq.~\ref{eq:lowlevelq}, with a linear learning rate. With probability at least $1 - \frac{\delta}{m}$ we have that for every state-joint action pair $|W^j_{t;\tau_k}(s, \boldsymbol{a})| \leq \frac{\psi}{2 + \psi} \beta D^j_k$, for any $t > \tau_{k+1}$ and any positive constant $\psi \leq 0.712$, i.e., 

\begin{equation}
    \begin{array}{l}
         Pr\Big[ \forall t \in [\tau_{k+1}, \tau_{k+2}]: W^j_{t;\tau_k} (s, \boldsymbol{a}) \leq \frac{\psi}{2 + \psi} \beta D^j_k \Big] \geq 1 - \frac{\delta}{m}
    \end{array}
\end{equation}

\noindent given that $\tau_k \geq \Theta(\frac{Q^{2,j}_{\max} \ln (Q^j_{\max} |S| \Pi_i |A|_i m)/(\psi \delta \beta D_k)}{\psi^2 \beta^2 D^2_k} )$.

\end{lemm}

\begin{proof}

By Lemma~\ref{lemma:wmartingalelinear} we can apply Azuma's inequality on the term, $W_{t;\tau_k}^{t-\tau_k +1,j}$ (note that $W_{t;\tau_k}^{t-\tau_k +1,j} = W_{t;\tau_k}$), and with the expression $c_i = \Theta \Big(\frac{Q^j_{\max}}{n(s,a,0,t)} \Big)$ for any $t \geq \tau_{k+1}$. Therefore, we derive that,

\begin{equation}
    \begin{array}{l}
         Pr[ |W^j_{t;\tau_k} \geq \Tilde{\epsilon} |] \leq 2e^{\frac{-2e\Tilde{\epsilon}^2}{\sum_{i=\tau_k}^t c^2_i}}
         \leq 2e^{-c\frac{\Tilde{\epsilon}^2 n(s, \boldsymbol{a}, \tau_k, t)}{Q^{2,j}_{\max}}}
    \end{array}
\end{equation}

\noindent for some positive constant c.

Let us define,

\begin{equation}
    \begin{array}{l}
         \zeta_t(s, \boldsymbol{a}) 
         = { 1, \textrm{ if } \alpha_t(s, \boldsymbol{a}) \neq 0}
         \\ 
         \zeta_t(s, \boldsymbol{a}) 
         = { 0, \textrm{ otherwise }}.
    \end{array}
\end{equation}

Using the union bound and the property that, in an interval of length $(1 + \psi) L\tau_k$, each state-joint action pair is visited at least $(1 + \psi) \tau_k$ times, we get 

\begin{equation}
\begin{array}{l}
     Pr\Big[\forall t \in [\tau_{k+1}, \tau_{k+2}]: |W^j_{t;\tau_k}(s, \boldsymbol{a}) | \geq \Tilde{\epsilon} \Big]
     \\ \\ 
     \leq Pr[ \forall t \geq ((1 + \psi) L + 1) \tau_k: |W^j_{t;\tau_k}(s, \boldsymbol{a})| \geq \Tilde{\epsilon}]
     \\ \\
     \leq 
     \sum^\infty_{t = ((1 + \psi)L + 1)\tau_k} Pr\Big[   |W^j_{t;\tau_k}(s, \boldsymbol{a}) | \geq \Tilde{\epsilon} \Big]
     \\ \\
     \leq 
     \sum^\infty_{t = ((1 + \psi)L + 1)\tau_k} \zeta_t(s, \boldsymbol{a})  2e^{-c\frac{\Tilde{\epsilon}^2 n(s, \boldsymbol{a}, 0, t)}{Q^{2,j}_{\max}}}
     \\ \\ 
      \leq 
     2e^{-c\frac{\Tilde{\epsilon}^2 ((1 + \psi)\tau_k) }{Q^{2,j}_{\max}}} \sum_{t=0}^\infty e^{-\frac{t\Tilde{\epsilon}^2}{ Q^{2,j}_{\max}}}
     \\ \\ 
     = \frac{2e^{-c\frac{\Tilde{\epsilon}^2 ((1 + \psi)\tau_k) }{Q^{2,j}_{\max}}} }{1 - e^{\frac{-\Tilde{\epsilon}^2}{Q^{2,j}_{\max}}}}
     \\ \\ 
     = \Theta(\frac{Q^{2,j}_{\max}e^\frac{-c'\tau_k\Tilde{\epsilon}^2}{Q^{2,j}_{\max}} }{\Tilde{\epsilon}^2})
     
\end{array}
\end{equation}

\noindent for some positive constant $c'$. Setting $\frac{\delta}{m |S|\Pi_i |A|_i} = \Theta \Big( \frac{e^-{\frac{c' \tau_k \Tilde{\epsilon}^2}{Q^{2,j}_{\max}}} Q^{2,j}_{\max}}{\Tilde{\epsilon}^2} \Big)$, which holds for $\tau_k = \Theta(\frac{Q^{2,j}_{\max} \ln(Q^j_{\max} |S| \Pi_i |A|_im /(\delta \Tilde{\epsilon}) )}{\Tilde{\epsilon}^2})$, and the expression, $\Tilde{\epsilon} = \frac{\psi}{2 + \psi} \beta D^j_k$ assures us that for every $t \geq \tau_{k+1}$ (and as a result for any $t \in [\tau_{k+1}, \tau_{k+2}]$), with probability at least $1 - \frac{\delta}{m}$ the statement holds at every state-joint action pair.

\end{proof}

We have bounded for each iteration the time needed to achieve the desired probability of $1 - \frac{\delta}{m}$. The following lemma provides a bound for the error in all the iterations. 

\begin{lemm}\label{lemma:everyiterationbound}
Consider the low-$Q$ update given in Eq.~\ref{eq:lowlevelq}, with a linear learning rate. With probability $1 - \delta$, for every iteration $k \in [1,m]$, time $t \in [\tau_{k+1}, \tau_{k+2}]$, and any positive constant $\psi \leq 0.712$, we have $|W^j_{t;\tau_k} | \leq \frac{\psi \beta D^j_k}{ 2 + \psi}$, i.e., 

\begin{equation}
\begin{array}{l}
        Pr\Big[\forall k \in [1,m], \forall t \in [\tau_{k+1}, \tau_{k+2}]:   |W^j_{t;\tau_k}| \leq \frac{\psi \beta D^j_k}{2 + \psi} \Big] 
    \geq 1 - \delta 
\end{array}
\end{equation}
given that $\tau_0 = \Theta \Big( \frac{Q^{2,j}_{\max} \ln(Q^j_{\max} |S| |A| m /(\delta \beta \epsilon \psi) )}{\psi^2 \beta^2 \epsilon^2} \Big)$

\end{lemm}

\begin{proof}

From Lemma~\ref{lemma:Wtlinearbound}, we know that 
\begin{equation}
    \begin{array}{l}
         Pr \Big[ \forall t \in [\tau_{k+1}, \tau_{k+2}]: |W^j_{t;\tau_k}| \geq \frac{\psi \beta D^j_k}{ 2 + \psi} \Big] \leq \frac{\delta}{m}.
    \end{array}
\end{equation}

Using the union bound, we have that, 

\begin{equation}
\begin{array}{l}
     Pr \Big[ \forall k \leq m, \forall t \in [\tau_{k+1}, \tau_{k+2}]: |W^j_{t;\tau_k}| \geq \frac{\psi \beta D^j_k}{ 2 + \psi} \Big]
     \\ \\ 
     
     \leq \sum_{k=1}^m Pr\Big[ \forall t \in [\tau_{k+1}, \tau_{k+2}] |W^j_{t; \tau_k}| \geq \frac{ \psi \beta D^j_k }{2 + \psi} \Big]
     \\ \\ 
     \leq \delta.
\end{array}
\end{equation}
\end{proof}

\begin{theorem2}
Let us specify that with probability at least $1 - \delta$, we have for an agent $j$, $||Q^j_T - Q^j_* ||_\infty \leq \epsilon$. The bound on the rate of convergence of low-$Q$, $Q^j_T$, with a linear learning rate is given by
\begin{equation*}
    T = \Omega \Big( (L + \psi L + 1)^{\frac{1}{\beta} \ln \frac{Q^j_{\max}}{\epsilon}} \frac{Q^{2,j}_{\max}\ln(\frac{|S|\Pi_i |A|_i Q^j_{\max}}{\delta \beta \epsilon \psi})}{\beta^2 \epsilon^2 \psi^2}   \Big),
\end{equation*}
\noindent where $\psi$ is a small arbitrary positive constant satisfying $\psi \leq 0.712$

\end{theorem2}

\begin{proof}

The Theorem~\ref{theorem:linearlearningrate} follows from Lemmas~\ref{lemma:everyiterationbound}, \ref{lemma:Ytlinearbound}, \ref{lemm:Dkbound}, and the fact that $a_{k+1} = a_k + (1 + \psi)La_k  = a_0 ((1 + \psi)L+1)^k$. 

Specifically, substitute the value of $k$ from Lemma~\ref{lemm:Dkbound} (lower bound for $m$), value of $a_0$ from Lemma~\ref{lemma:everyiterationbound} (value of $\tau_0$), and see that Lemma~\ref{lemma:Ytlinearbound} and Lemma~\ref{lemma:everyiterationbound} satisfy the condition for the lower bound in Lemma~\ref{lemma:boundlemma}.

\end{proof}

\section{Frequency Of Listening To Advisors}\label{sec:frequency}

This section plots the frequency of listening to each advisor in some of our experiments. We would like to show that the MA-TLQL listens more to the good advisor and avoids the bad advisor more than other related baselines. For these experiments, we consider the TLQL \citep{li2019two} and ADMIRAL-DM \citep{wang2017improving}
algorithms for comparison. Since the CHAT implementation uses the same method to choose advisors as ADMIRAL-DM (weighted random policy approach), we will omit CHAT for these results (performance is similar to ADMIRAL-DM). DQfD uses pretraining and does not choose advisors in an online fashion; hence we omit DQfD for these experiments as well.

\begin{figure}
	\subfloat[Good Advisor (Advisor~1)]{{\includegraphics[width=0.45\textwidth]{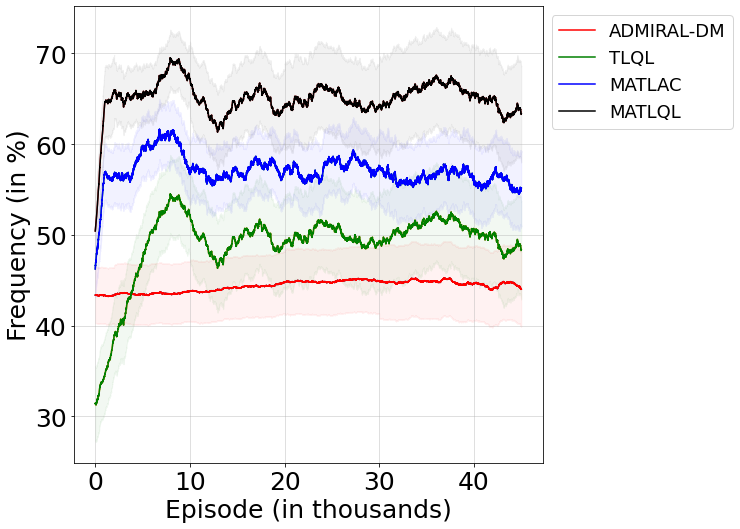}}}
	\\
		\subfloat[Bad Advisor (Advisor~4)]{{\includegraphics[width=0.45\textwidth]{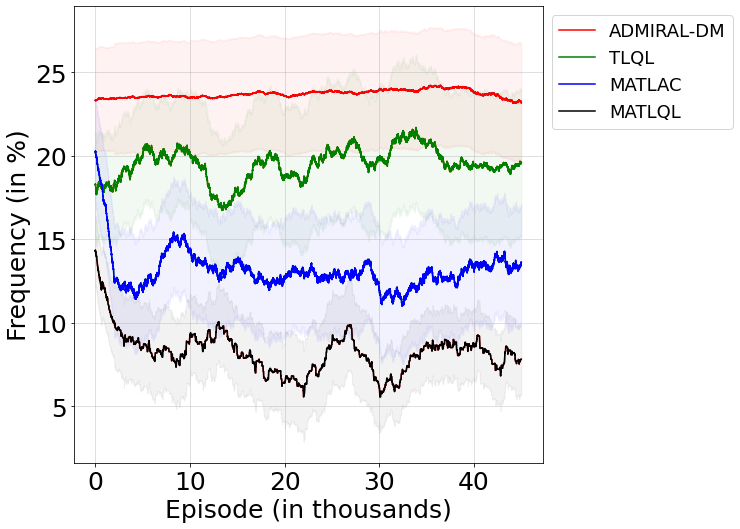}}}

  \caption{Frequency of listening to advisors in the two-agent Pommerman experiment with four sufficient advisors of different quality (Experiment~1)}%
	\label{fig:onevsonefrequency}
\end{figure}

We consider our first experiment in Section~\ref{sec:experiments} (Experiment~1), where we had a set of four sufficient advisors of different quality. The first advisor (Advisor~1) had a better quality than the others, and the agents must listen more to this advisor. On the other hand, Advisor~4 only suggested random actions and the agents are expected to avoid listening to this advisor. In Figure~\ref{fig:onevsonefrequency}(a), we plot a curve that corresponds to the percentage of time steps an algorithm listened to Advisor~1 out of all the time steps the algorithm had an oppourtunity to listen to one of the available advisors. From the plots, we see that MA-TLQL listens more (compared to the other baselines) to this advisor (Advisor~1) from the beginning until the end of training. Since the MA-TLQL uses an ensemble technique to choose an advisor, this gives it a distinct advantage in the early stages of training. Further, since MA-TLQL performs an explicit evaluation of the advisors independent of the RL policy, it manages to listen more to the correct advisor as compared to other baselines. MA-TLAC also listens more to the good advisor as compared to the other baselines (TLQL, ADMIRAL-DM). Since TLQL couples the advisor evaluation with the RL policy, it listens a lot less to the good advisor as compared to MA-TLQL. Also, TLQL considers the RL policy as part of the high-level table, which makes it less reliant on advisors. This could be a problem when good advisors are available. In Figure~\ref{fig:onevsonefrequency}(b), we plot the percentage of each algorithm listening to the bad advisor (Advisor~4). We see that MA-TLQL has the least dependence on this advisor as compared to all other algorithms. This reinforces our observation that MA-TLQL is most likely to choose to listen to the correct advisors in this multi-agent setting.

\begin{figure}
	\subfloat[Good Advisor (Advisor~1)]{{\includegraphics[width=0.45\textwidth]{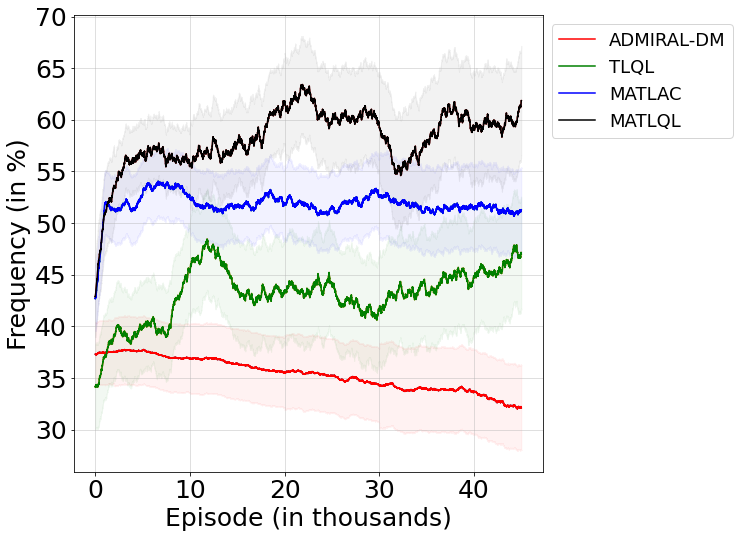}}}
	\quad \quad
		\subfloat[Bad Advisor (Advisor~4) ]{{\includegraphics[width=0.45\textwidth]{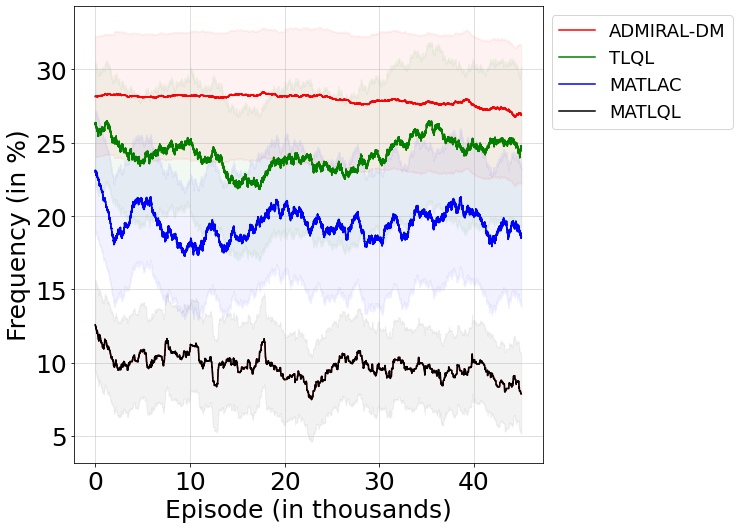}}}

  \caption{Frequency of listening to advisors in the two-agent Pommerman experiment with four insufficient advisors of different quality (Experiment~3)}%
	\label{fig:onevsoneinsufficientfrequency}
\end{figure}

We plot the percentage of listening to Advisor~1 and Advisor~4 in Experiment~3 from Section~\ref{sec:experiments} that had an insufficient set of four advisors of decreasing quality. The results in Figure~\ref{fig:onevsoneinsufficientfrequency}(a) and (b) show that MA-TLQL listens more to the good advisor and less to the bad advisor, same as our observations for Experiment~1.

\begin{figure}
	\subfloat[Good Advisor (Advisor~1)]{{\includegraphics[width=0.45\textwidth]{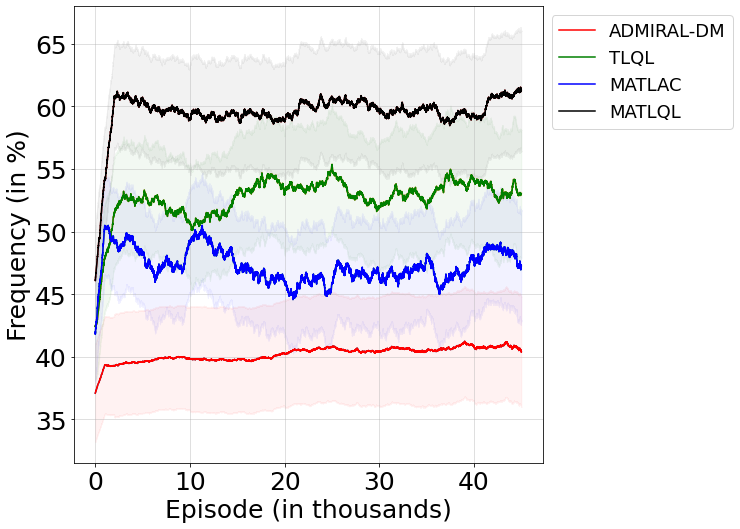}}}
	\quad \quad
		\subfloat[Bad Advisor (Advisor~4)]{{\includegraphics[width=0.45\textwidth]{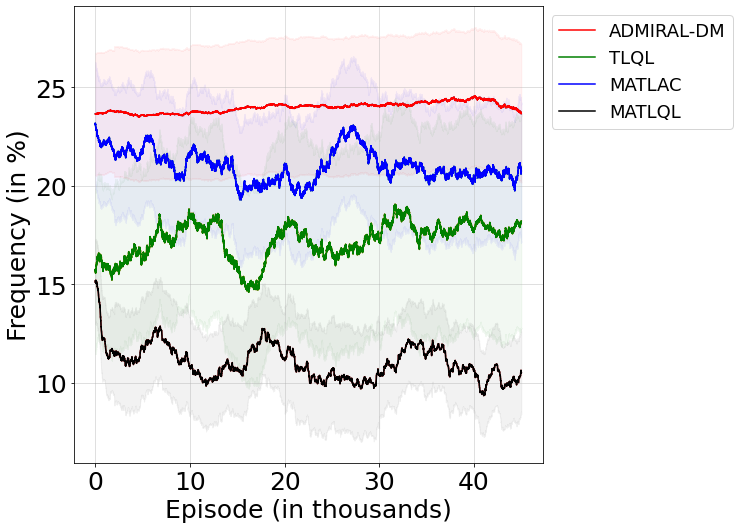}}}

  \caption{Frequency of listening to advisors in the team Pommerman experiment with four sufficient advisors of different quality (Experiment~5)}%
	\label{fig:teamcompsufficientfrequency}
\end{figure}

Similarly, we plot the percentages of listening to the good and bad advisor in the Pommerman team environment (mixed setting) used in Experiment~5 in Figure~\ref{fig:teamcompsufficientfrequency}(a) and (b). Here we plot the results for one of the two pommerman agents playing in the same team (the other agent's results are similar). Once again, we note that MA-TLQL listens more to the correct advisor than the other algorithms and better avoids the bad advisors compared to the other algorithms.

\section{Nature Of Algorithms Considered}\label{sec:natureofalgorithms}

\begin{table}
\centering
\renewcommand{\arraystretch}{2}
\begin{tabular}{||p{0.20 \linewidth}  | p {0.20\linewidth}  | p {0.22\linewidth}  | p {0.24
\linewidth} ||} 
 \hline\hline
 Algorithm & Nature of Algorithm & Number of Advisors & Type of Demonstrations  \\ [0.5ex] 
 \hline\hline
 DQN \cite{mnih2015human}  & Independent  & Does not learn from advising & Not applicable  \\ 
 \hline
 DQfD \cite{hester2018deep}  & Independent  & One & Offline  \\ 
 \hline
  CHAT \cite{wang2017improving} & Independent  & One & Online  \\ 
   \hline
  ADMIRAL-DM \cite{Subramanian2022multiagent} & Multi-agent  & One & Online \\ 
 \hline
  TLQL \cite{li2019two} & Independent  & More than one & Online \\ 
 \hline
MA-TLQL (ours)  & Multi-agent  & More than one & Online \\ [1ex] 
 \hline
 MA-TLAC (ours)  & Multi-agent  & More than one & Online \\ [1ex] 
 \hline
 \hline
\end{tabular}
\caption{Description of all algorithms considered in this paper}
\label{table:baseline}
\end{table}

In Table~\ref{table:baseline}, we tabulate all the algorithms considered in this paper. The differences between the algorithms stem from the nature of the algorithm (independent or multi-agent), ability to naturally support learning from conflict demonstrations (i.e., more than one advisor), and type of demonstrations that they naturally support (offline vs. online). Offline demonstrations are demonstrations that are collected (in a memory buffer) from an advisor before the ``training phase'' where the algorithm is trained using interactions with the environment. These demonstrations are typically used to train the algorithm in a ``pre-training'' phase before regular training, as done in several prior works \cite{kim2013learning, hester2018deep, gao2018reinforcement}. Alternatively, online demonstrations are obtained in real-time during the training phase (and not pre-collected). Here, an agent can actively obtain action recommendations directly from an advisor for the current game context.

In general, in all the experiments considered in this paper, we found that algorithms that consider actions of other agents to provide best-responses (non-independent) performed better than independent algorithms which consider all other agents to be part of the environment. One reason for this behaviour could be the fact that independent algorithms break the Markovian assumptions in reinforcement learning methods \cite{tan1993multi}. Additionally, we found that algorithms that learn from online advising perform better as compared to algorithms that learn from offline advising. If algorithms learn from online advising, then these algorithms can exploit the knowledge of advisors in response to dynamically changing other agent(s), in real-time. Since other agent(s)/opponent(s) are not typically available before training in multi-agent environments, offline demonstrations are not very successful in multi-agent training, as opposed to single-agent training where they were quite successful \cite{hester2018deep}. In this paper, we considered environments where the advising come from multiple sources of independent knowledge. In many states, the different advisors provide conflicting recommendations. Hence, algorithms that can effectively resolve conflicting information from the different advisors were successful as compared to other algorithms that are not capable of naturally supporting learning under multiple conflicting advisors. As seen from Table~\ref{table:baseline}, the only two algorithms that have all the three desirable properties (i.e, support multi-agent update, learn from conflicting advisors, support offline demonstrations) are MA-TLQL and MA-TLAC which gave the best performance in most of the seven experiments considered in Section~\ref{sec:experiments}. 

In our paper, we do not consider HMAT \cite{Kim2020Learning} and LeCTR \cite{omidshafiei2019learning} as baselines, though these can also be classified as action advising methods. We do not consider these algorithms as appropriate benchmarks due to two important reasons. First, they are both restricted to two-agent cooperative settings while we are interested in more general domains including those with more than two agents and both competitive and mixed-motive environments. Second, we are interested in independent learning from a set of external advisors, while both HMAT and LeCTR focus on peer-to-peer learning (partly due to their focus on teaching teams, as opposed to our work which is on more general multi-agent learning problems).

\section{MA-TLQL With Learning Advisors}\label{sec:learningadvisors}

\begin{figure}[t]
	\subfloat[Pursuit Environment ]{{\includegraphics[width=0.44\textwidth]{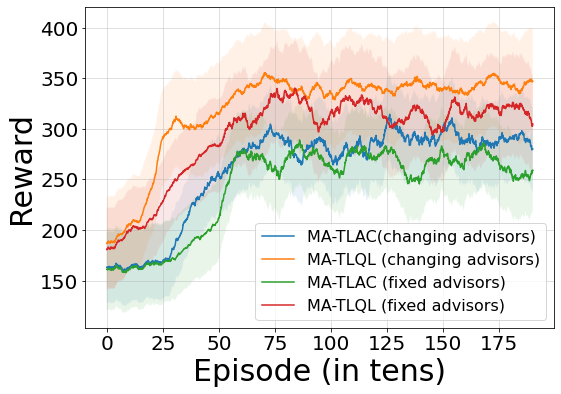}}} 
	\\
			\subfloat[Predator-Prey Environment]{{\includegraphics[width=0.44\textwidth]{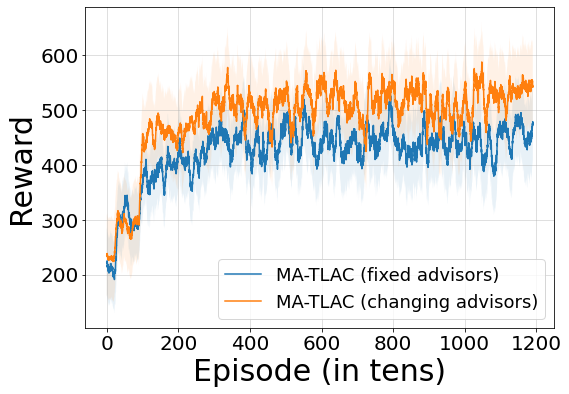}}}

  \caption{Performances of MA-TLQL and MA-TLAC under changing (learning) advisors and fixed advisors. Here (a) corresponds to the setting in Experiment 6 and (b) corresponds to the setting in Experiment 7, described in Section~\ref{sec:experiments}.}%
	\label{fig:learningadvisor}
\end{figure}

In all the experiments in this paper we considered advisors that are fixed and non-changing, which we specified in Section~\ref{sec:extending}. If the advisors are fixed, their $Q$-values can be determined using the high-$Q$ updates in MA-TLQL. The fixed nature of advisors allows us to provide theoretical guarantees of convergence of the high-$Q$ values (under the assumption of infinite updates in the limit) using similar arguments as in Theorem~\ref{theorem:lowqconvergence}. However, experimentally we can still consider other types of advisors which are actively learning during the training stage and hence are updating the policies actively (though using such advisors do not have any theoretical guarantees of convergence). In this section, we will revisit Experiment~6 and Experiment~7 from Section~\ref{sec:experiments} and study the performances of MA-TLQL and MA-TLAC under learning advisors. For statistical significance we use the unpaired 2-sided t-test and report $p$-values, as in the earlier experiments.

First we consider Experiment~6, where we used the Pursuit cooperative environment along with four pre-trained networks of DQN as the advisor. Now, we will consider the same set of four advisors and label them as ``fixed advisors''. Additionally, we will let the same four pre-trained networks of DQN continue training while it is being actively used for action advising. We label this set of four advisors as ``changing advisors''. In Figure~\ref{fig:learningadvisor}(a) we plot the performances (plot of episodic rewards) of MA-TLQL and MA-TLAC along with the fixed as well as the changing advisor set. While using the changing advisor set, we see that MA-TLQL and MA-TLAC outperform their counterparts using the fixed advisors ($p < 0.04$). As the advisors are actively learning and changing their strategies during the training stage, they are able to provide better action recommendations to MA-TLQL and MA-TLAC at the different parts of the environment. This is reflected in their superior performances. 

Similarly, we revisit Experiment~7 with the Predator-Prey environment and consider two different sets of advisors. In Experiment~7, we used four pre-trained DQN networks as advisors. We will reuse the same set of four advisors and label them as the ``fixed advisors''. A continuously training counterpart is labelled as the ``changing advisors''. We plot the performances of MA-TLAC in the Predator-Prey experiment in Figure~\ref{fig:learningadvisor}(b), and once again we notice that MA-TLAC using the changing advisor set outperforms MA-TLAC using the fixed advisor set ($p < 0.03$).  

Hence, experimentally we see that MA-TLQL and MA-TLAC can still be used with changing advisors which provide good empirical performances. However, the non-stationary nature of these advisors makes it impossible to provide theoretical guarantees of convergence. This is similar to using independent algorithms in multi-agent environments, which do not have any theoretical guarantees of convergence to either a local or global optimum, yet empirically, several prior works have noted that these algorithms perform well in various multi-agent environments \cite{hernandez2019survey, matignon2012independent}. 

\section{MA-TLQL With Opponent Modelling}\label{sec:opponentmodelling}

\begin{figure}[ht]
	\subfloat[One vs. One Pommerman: Insufficient different quality advisors]{{\includegraphics[width=0.44\textwidth]{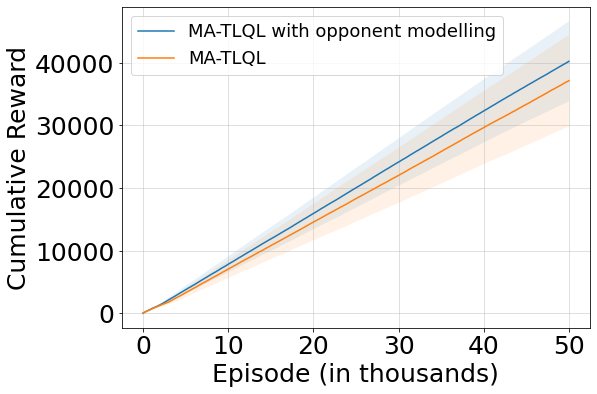}}} 
	\\  
			\subfloat[One vs. One Pommerman: Insufficient similar quality advisors]{{\includegraphics[width=0.44\textwidth]{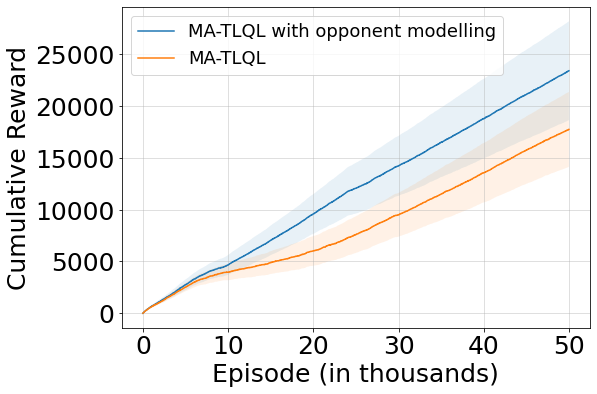}}}

  \caption{Performances of MA-TLQL with and without opponent modelling using four insufficient advisors. Here (a) corresponds to the setting in Experiment~3 and (b) corresponds to the setting in Experiment~4, described in Section~\ref{sec:experiments}.}%
	\label{fig:MA-TLQLopponentmodelling}
\end{figure}

In all the experiments in this paper, the MA-TLQL used a multi-agent update where it simply tracked opponent actions (by considering the previous action) and did not perform any active opponent modelling. The core focus of this paper is learning from advisors, and we chose not to perform any particular opponent modelling technique to keep the algorithm simple. However, in this section, we will implement MA-TLQL along with an opponent modelling technique that uses a separate neural network (2 Relu layers of 50 neurons and an output layer) to predict the action of the opponent. The network uses the state and previous action of the opponent as the input and predicts the action of the opponent. This predicted action of the opponent is used by the agent to calculate best responses. Finally, the actual observed action of the opponent is used to define a cross-entropy loss function that is used to train the network. We revisit Experiment~3 and Experiment~4 in Section~\ref{sec:experiments} where we used four insufficient advisors of different and similar quality. Now, we will use the same experimental procedure (with the same set of advisors) and consider the performances of MA-TLQL with and without active opponent modelling. For statistical significance we use the unpaired 2-sided t-test and report $p$-values.

We plot the performances in Figure~\ref{fig:MA-TLQLopponentmodelling}. From both Figure~\ref{fig:MA-TLQLopponentmodelling}(a) and (b) we see that MA-TLQL with opponent modelling performs better than the MA-TLQL algorithm that does not perform any active opponent modelling. In the Figure~\ref{fig:MA-TLQLopponentmodelling}(a) we see that opponent modelling only provides a small increase in performance ($p < 0.2$). One reason for this is the fact that MA-TLQL without opponent modelling is already showing a high performance in this experiment by learning efficiently from the advisors. Alternatively, Figure~\ref{fig:MA-TLQLopponentmodelling}(b) shows a considerable increase in performance while MA-TLQL is using opponent modelling ($p < 0.04$). Here, the performance of MA-TLQL without opponent modelling is not as good is the performance in Figure~\ref{fig:MA-TLQLopponentmodelling}(a) and opponent modelling shows a marked improvement in performance.

\section{MA-TLQL With Different Numbers Of Advisors}\label{sec:differentadvisors}

\begin{figure}[ht]
    \centering
    \includegraphics[width=0.44 \textwidth]{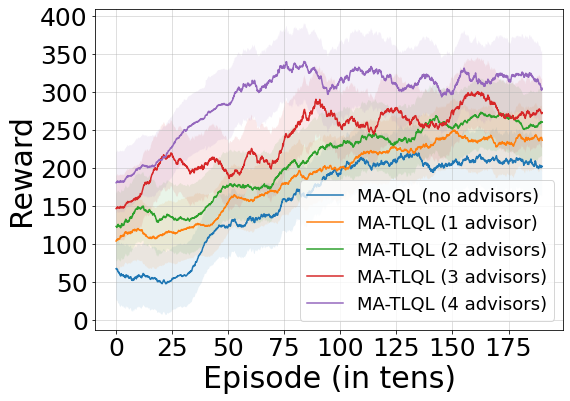}
    \caption{MA-TLQL with different numbers of advisors in the Pursuit environment (setting used in Experiment~6, described in Section~\ref{sec:experiments})}
    \label{fig:matlqldifferentnumbers}
\end{figure}

In this section we train MA-TLQL with different number of advisors in the Pursuit environment considered in Experiment~6 of Section~\ref{sec:experiments}. Recall that in Experiment~6 we considered four pre-trained networks of DQN as the advisor (pre-trained for 500, 1000, 1500, and 2000 episodes). Labelling these advisors, we will denote Advisor~1 as the advisor pre-trained for 500 episodes, Advisor~2 as the advisor pre-trained for 1000 episodes, Advisor~3 as the advisor pre-trained for 1500 episodes, and Advisor~4 as the advisor pre-trained for 2000 episodes. In this experiment, we initially train MA-TLQL with no advisor (denoted as MA-QL), and subsequently train MA-TLQL with the addition of one advisor from the set of advisors. MA-QL always chooses actions from the low-$Q$ (since there are no advisors, the high-$Q$ is not maintained). The objective is to study the performance of MA-TLQL under the presence of different numbers of advisors. For statistical significance we use the unpaired 2-sided t-test and report $p$-values.

The results are given in Figure~\ref{fig:matlqldifferentnumbers}. We see that MA-QL trained without any advisor gives the least performance. MA-TLQL trained using one advisor (Advisor~1) gives a better performance than MA-QL ($p < 0.06$). Next, we train MA-TLQL with two advisors (Advisor~1 and Advisor~2), which gives a better performance than MA-QL and MA-TLQL trained with one advisor ($p < 0.1$). Similarly, MA-TLQL with three advisors (Advisor~1, Advisor~2, and Advisor~3) gives a better performance than the case with two advisors ($p < 0.09$) and the best performance is given by MA-TLQL with all four advisors ($p < 0.04$). This result shows that MA-TLQL performance can keep improving with the addition of better advisors (than those available in the current set), which shows that MA-TLQL is capable of identifying the right advisor from the available set and exploiting the expertise of different available advisors in a multi-agent environment. From the $p$-values, we note that the observation of the best performance of MA-TLQL with all four advisors is statistically significant. While we observe constantly improving performances with each advisor, some of these comparisons are not statistically significant as provided by the $p$-values.

\section{Illustrative Example}\label{sec:illustrativeexample}

In this section we would like to show a toy example where the TLQL updates as provided by Li et al.~\cite{li2019two} takes a longer time to figure out the best advisor from a set of advisors, as compared to the MA-TLQL updates, we introduced in this paper. We will just use a single agent based grid-world environment as given in Figure~\ref{fig:counterexampleimages}, instead of multi-agent environments. The TLQL updates will use the Bellman update for the low-$Q$ updates and a subsequent synchronization step to update its high-$Q$ values as proposed in Li et al.~\cite{li2019two}. In MA-TLQL, the low-$Q$ values will use the control update as given in Eq.~\ref{eq:control} albeit in a single-agent fashion (joint actions need not be considered), since this environment only contains one agent. The high-$Q$ will use the evaluation update as given in Eq.~\ref{eq:evaluation} in a single-agent fashion. Also, for simplicity, we do not use the ensemble selection strategy in Eq.~\ref{eq:valueofvote} for the MA-TLQL action selection in this example. Rather we only use the high-$Q$ and low-$Q$ updates for advisor and action evaluation at the given state. Here we have a set of 6 states $\{S1,S2, S3, S4, S5, G\}$ with the agent starting at state $S1$ and trying to reach the goal state $G$. The states $S3, S4$ and $G$ are the terminal states where the agent receives a reward of -1 in states $S3$ and $S4$, and a reward of +1 in state $G$. The state $S5$ only exists for symmetry (cannot be reached in practice). The agent can take one of the two actions $\{R, D\}$ (to denote right and down respectively) at each state. In this environment, it can be seen that the agent needs to take action $R$ in states $S1$ and $S2$ to obtain the maximum rewards. The agent has access to two advisors $A1$ and $A2$, where $A1$ is an optimal advisor providing the correct action (right) at every state and $A2$ is a sub-optimal advisor which provides action $R$ with probability 0.5 and the action $D$ with probability 0.5. All transitions in this domain are deterministic. Also, we specify that the learning rate ($\alpha$) is 0.1 and the discount factor ($\gamma$) is 0.9.

\begin{figure}
    \centering
    \includegraphics[width=0.45 \textwidth]{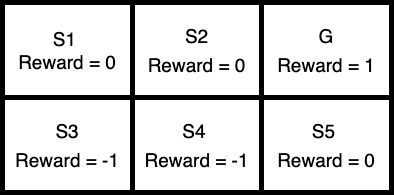}
    \caption{Toy environment to compare updates of TLQL and MA-TLQL}
    \label{fig:counterexampleimages}
\end{figure}

Now we consider the $Q$-updates (high-level and low-level) pertaining to both MA-TLQL and TLQL. At the beginning, we initialize all the $Q$-values to 0 arbitrarily.

At the initial time step ($t = 0$) let us assume that the agent starts at the initial state $S1$. Let both the advisors suggest action $R$. Then the agent takes this action and first updates its low-$Q$ using the equation, 

\begin{equation}
    \begin{array}{l}
lowQ_1(S1, R) 
\\ \quad =  lowQ_0(S1, R)   +  \alpha \Big(r + \gamma \max_a lowQ_0(S2, a) - lowQ_0(S1, R) \Big)
\\ \\ 
lowQ_1(S1, R) = 0   +  0.1 \Big(0 + 0.9 \times 0 - 0 \Big)
\\ \\ 
lowQ_1(S1, R) = 0.
    \end{array}
\end{equation}

\begin{table}
\subfloat[Low-$Q$ values at time $t=1$]{
\begin{tabular}{||c | c | c||} 
 \hline
 (ST./AC.) & Right & Down \\ [0.5ex] 
 \hline\hline
 S1 & 0 & 0 \\ 
 \hline
 S2 & 0 & 0 \\ [1ex] 
 \hline
\end{tabular}
} 
\subfloat[TLQL high-$Q$ at time $t = 1$]{
\begin{tabular}{||c | c | c||} 
 \hline
 (ST./AC.) & Adv. 1 & Adv. 2 \\ [0.5ex] 
 \hline\hline
  S1 & 0 & 0 \\ 
 \hline
 S2 & 0 & 0 \\ [1ex] 
 \hline
\end{tabular}
}\\
\subfloat[MA-TLQL high-$Q$ at time $t = 1$]{
\begin{tabular}{||c | c | c||} 
 \hline
 (ST./AC.) & Adv. 1 & Adv. 2 \\ [0.5ex] 
 \hline\hline
 S1 & 0 & 0 \\ 
 \hline
 S2 & 0 & 0 \\ [1ex] 
 \hline
\end{tabular}
}
\caption{TLQL and MA-TLQL updates at time $t = 1$. The columns refer to the actions and the rows refer to the states.}
\label{tab:qatt0}
\end{table}

Now, TLQL will set the value of the high-$Q$ of both advisors to be 0 (synchronization). The MA-TLQL updates will also yield a value of 0 for both the advisors. In Table~\ref{tab:qatt0} we tabulate the $Q$-values for the non-terminal states ($S1$ and $S2$). 

At the next time step $t=2$, the agent is at state $S2$. Again let both advisors suggest the right action. Now the low-$Q$ is updated as, 

\begin{equation}
    \begin{array}{l}
lowQ_2(S2, R) \\ \\ =  lowQ_1(S2, R)   +  \alpha \Big(r + \gamma \max_a lowQ_1(G, a) - lowQ_1(S2, R) \Big)
\\ \\ 
lowQ_2(S2, R) = 0   +  0.1 \Big(1 + 0.9 \times 0 - 0 \Big)
\\ \\ 
lowQ_2(S2, R) = 0.1
    \end{array}
\end{equation}

We specify that when the next state is terminal, the temporal difference (T.D.) target is the reward itself. Both the algorithms, TLQL and MA-TLQL, will have the same high-$Q$ values in this case as well. The $Q$-values for time $t=2$ are tabulated in Table~\ref{tab:qatt1}.

\begin{table}
\subfloat[Low-$Q$ values at time $t=2$]{
\begin{tabular}{||c | c | c||} 
 \hline
 (ST./AC.) & Right & Down \\ [0.5ex] 
 \hline\hline
  S1 & 0 & 0 \\ 
 \hline
 S2 & 0.1 & 0 \\ [1ex] 
 \hline
\end{tabular}
}
\subfloat[TLQL high-$Q$ at time $t = 2$]{
\begin{tabular}{||c | c | c||} 
 \hline
 (ST./AC.) & Adv. 1 & Adv. 2 \\ [0.5ex] 
 \hline\hline
 S1 & 0 & 0 \\ 
 \hline
 S2 & 0.1 & 0.1 \\ [1ex] 
 \hline
\end{tabular}
}\\
\subfloat[MA-TLQL high-$Q$ at time $t = 2$]{
\begin{tabular}{||c | c | c||} 
 \hline
 (ST./AC.) & Adv. 1 & Adv. 2 \\ [0.5ex] 
 \hline\hline
  S1 & 0 & 0 \\ 
 \hline
 S2 & 0.1 & 0.1 \\ [1ex] 
 \hline
\end{tabular}
}
\caption{TLQL and MA-TLQL updates at time $t = 2$}
\label{tab:qatt1}
\end{table}

Now, we move to time $t=3$. Since the goal has been reached at the previous time step, the agent resets back to the initial state $S1$. Now, let us assume that both advisors specify the right action again at state $S1$. The low-$Q$ values are updated as,

\begin{equation}
    \begin{array}{l}
lowQ_3(S1, R) \\ \\  =  lowQ_2(S1, R)   +  \alpha \Big(r + \gamma \max_a lowQ_2(S2, a) - lowQ_2(S2, R) \Big)
\\ \\ 
lowQ_3(S1, R) = 0   +  0.1 \Big(0 + 0.9 \times 0.1 - 0 \Big)
\\ \\ 
lowQ_3(S1, R) = 0.009
    \end{array}
\end{equation}

Again, high-$Q$ values for MA-TLQL and TLQL will be the same as the low-$Q$ values and are tabulated in Table~\ref{tab:qatt2}.

\begin{table}
\subfloat[Low-$Q$ values at time $t=3$]{
\begin{tabular}{||c | c | c||} 
 \hline
 (ST./AC.) & Right & Down \\ [0.5ex] 
 \hline\hline
  S1 & 0.009 & 0 \\ 
 \hline
 S2 & 0.1 & 0 \\ [1ex] 
 \hline
\end{tabular}
}
\subfloat[TLQL high-$Q$ at time $t = 3$]{
\begin{tabular}{||c | c | c||} 
 \hline
 (ST./AC.) & Adv. 1 & Adv. 2 \\ [0.5ex] 
 \hline\hline
  S1 & 0.009 & 0.009 \\ 
 \hline
 S2 & 0.1 & 0.1 \\ [1ex] 
 \hline
\end{tabular}
} \\
\subfloat[MA-TLQL high-$Q$ at time $t = 3$]{
\begin{tabular}{||c | c | c||} 
 \hline
 (ST./AC.) & Adv. 1 & Adv. 2 \\ [0.5ex] 
 \hline\hline
S1 & 0.009 & 0.009 \\ 
 \hline
 S2 & 0.1 & 0.1 \\ [1ex] 
 \hline
\end{tabular}
}
\caption{TLQL and MA-TLQL updates at time $t = 3$}
\label{tab:qatt2}
\end{table}

Next, the agent moves to state $S2$. We are at time $t=4$. Let the advisor $A2$ specify action $D$ at this state (Advisor $A1$ always specifies $R$). Also let us assume that the agent chooses to listen to $A2$ at this state ($Q$-values of both advisors are the same, so the agent is indifferent between the two advisors) and hence it performs action $D$ (down) from $A2$. Now the Q-value for the low-$Q$ can be updated as,

\begin{equation}
    \begin{array}{l}
lowQ_4(S2, D) =  lowQ_4(S2, D)   \\ \quad \quad \quad \quad \quad \quad \quad  +  \alpha \Big(r + \gamma \max_a lowQ_4(S4, a) - lowQ_4(S2, D) \Big)
\\ \\ 
lowQ_4(S2, D) = -0.1
    \end{array}
\end{equation}

The high-$Q$ values for the MA-TLQL update is given by, 

\begin{equation}
    \begin{array}{l}
highQ_4(S2, A2) =  highQ_3(S2, A2)  \\ \quad \quad \quad \quad \quad +  \alpha \Big(r + \gamma highQ_3(S4, A2) - highQ_3(S2, A2) \Big)
\\ \\ 
highQ_4(S2, A2) = 0.1 + 0.1 (-1  - 0.1) 
\\ \\ 
highQ_4(S2, A2) = -0.01
    \end{array}
\end{equation}

All $Q$-values for time $t=4$ are tabulated in Table~\ref{tab:qatt3}.

\begin{table}
\subfloat[Low-$Q$ values at time $t=4$]{
\begin{tabular}{||c | c | c||} 
 \hline
 (ST./AC.) & Right & Down \\ [0.5ex] 
 \hline\hline
  S1 & 0.009 & 0 \\ 
 \hline
 S2 & 0.1 & -0.1 \\ [1ex] 
 \hline
\end{tabular}
}
\subfloat[TLQL high-$Q$ at time $t = 4$]{
\begin{tabular}{||c | c | c||} 
 \hline
 (ST./AC.) & Adv. 1 & Adv. 2 \\ [0.5ex] 
 \hline\hline
  S1 & 0.009 & 0.009\\ 
 \hline
 S2 & 0.1 & -0.1 \\ [1ex] 
 \hline
\end{tabular}
}\\
\subfloat[MA-TLQL high-$Q$ at time $t = 4$]{
\begin{tabular}{||c | c | c||} 
 \hline
 (ST./AC.) & Adv. 1 & Adv. 2 \\ [0.5ex] 
 \hline\hline
S1 & 0.009 & 0.009 \\ 
 \hline
 S2 & 0.1 & -0.01 \\ [1ex] 
 \hline
\end{tabular}
}
\caption{TLQL and MA-TLQL updates at time $t = 4$}
\label{tab:qatt3}
\end{table}

Since the state $S4$ was a terminal state, the agent is back to state $S1$. We are at time $t= 5$. At this state, let us assume that both the advisors specify action $R$. Now the agent chooses to perform this action and updates its low-$Q$ using, 

\begin{equation}
    \begin{array}{l}
lowQ_5(S1, R) =  lowQ_4(S1, R)   \\ \quad \quad \quad \quad +  \alpha \Big(r + \gamma \max_a lowQ_4(S2, a) - lowQ_4(S2, R) \Big)
\\ \\ 
lowQ_5(S1, R) = 0.009   +  0.1 \Big(0 + 0.9 \times 0.1 - 0.009 \Big)
\\ \\ 
lowQ_5(S1, R) = 0.0171.
    \end{array}
\end{equation}

Since both the advisors specified action $R$, both advisors are assigned the same $Q$-values in the high-$Q$ in the TLQL update. Now, the high-$Q$ estimate of advisor $A1$ will be the same as the low-$Q$ value in the MA-TLQL update as well. However, the high-$Q$ estimate of the advisor $A2$ of the MA-TLQL update will be as follows,

\begin{equation}
    \begin{array}{l}
highQ_5(S1, A2) =  highQ_4(S1, A2)  \\ \quad \quad \quad \quad \quad \quad +  \alpha \Big(r + \gamma highQ_4(S2, A2) - highQ_4(S1, A2) \Big)
\\ \\ 
highQ_5(S1, A2) = 0.009   +  0.1 \Big(0 + 0.9 \times -0.01 - 0.009 \Big)
\\ \\ 
highQ_5(S1, A2) = 0.009 - 0.0018 = 0.0072
    \end{array}
\end{equation}

\begin{table}
\subfloat[Low-$Q$ values at time $t=5$]{
\begin{tabular}{||c | c | c||} 
 \hline
 (ST./AC.) & Right & Down \\ [0.5ex] 
 \hline\hline
  S1 & 0.0171 & 0 \\ 
 \hline
 S2 & 0.1 & -0.1 \\ [1ex] 
 \hline
\end{tabular}
}
\subfloat[TLQL high-$Q$ at time $t = 5$\label{tab:TLQLupdatelast}]{
\begin{tabular}{||c | c | c||} 
 \hline
 (ST./AC.) & Adv. 1 & Adv. 2 \\ [0.5ex] 
 \hline\hline
  S1 & 0.0171 & 0.0171\\ 
 \hline
 S2 & 0.1 & -0.1 \\ [1ex] 
 \hline
\end{tabular}
}\\
\subfloat[MA-TLQL high-$Q$ at time $t = 5$\label{tab:MATLQLupdatelast}]{
\begin{tabular}{||c | c | c||} 
 \hline
 (ST./AC.) & Adv. 1 & Adv. 2 \\ [0.5ex] 
 \hline\hline
S1 & 0.0171 & 0.0072 \\ 
 \hline
 S2 & 0.1 & -0.01 \\ [1ex] 
 \hline
\end{tabular}
}
\caption{TLQL and MA-TLQL updates at time $t = 5$}
\label{tab:qatt4}
\end{table}

At this stage (time $t=5$) all the $Q$-values are tabulated in Table~\ref{tab:qatt4}. Comparing the high-$Q$ estimate of TLQL and MA-TLQL, we see that in the TLQL updates (Table~\ref{tab:TLQLupdatelast}), the agent is indifferent between following advisor $A1$ or advisor $A2$ in state $S1$ (same $Q$-values), while it would decide to choose advisor $A1$ at the state $S2$. Even after 5 update steps TLQL has not been able to determine the right advisor (as advisor $A1$ is better than $A2$) for both states. In contrast, the MA-TLQL updates found in Table~\ref{tab:MATLQLupdatelast} clearly show a higher $Q$-value for the advisor $A1$ than the advisor $A2$ for both the states. This example presents a situation where the MA-TLQL distinguishes between a good and a bad advisor faster than the vanilla TLQL update as introduced by Li et al.~\cite{li2019two}.

\section{Experimental Details}\label{sec:experimentaldetails}

This section provides the complete details of all of our experimental domains, including details about the reward function and the advisors used. For the Pommerman and Pursuit environments, we assume that all the actions of other agents are either directly observable (fully observable), shared amongst agents, or provided by the game engine to perform centralized updates. For the MPE environment, the actions of other agents are observable only during training and not during execution (CTDE). 

Table~\ref{table:description} contains a summary of all of our experimental settings along with the associated configuration of advisors.

\begin{table}[ht]
\centering
\renewcommand{\arraystretch}{2}
\begin{tabular}{||p{0.05 \linewidth}  |p{0.15 \linewidth} | p {0.17\linewidth}  | p {0.25\linewidth}  | p {0.09
\linewidth} ||} 
 \hline\hline
 Exp. & Domain & Type & Advisors & \# of training agents \\ [0.5ex] 
 \hline\hline
 1  & Two-agent Pommerman & Competitive & 4 sufficient advisors with different quality & 2 \\ 
 \hline
 2  & Two-agent Pommerman & Competitive & 4 sufficient advisors with similar quality & 2 \\ 
 \hline
  3 & Two-agent Pommerman & Competitive & 4 insufficient advisors with different quality & 2 \\ 
   \hline
  4 & Two-agent Pommerman & Competitive & 4 insufficient advisors with similar quality & 2 \\ 
 \hline
  5 & Four-agent Pommerman & Mixed & 4 sufficient advisors with different quality & 4 \\ 
 \hline
6  & Pursuit SISL & Cooperative & 4 insufficient advisors with different quality & 8\\ [1ex] 
 \hline
 7  & Predator-Prey MPE & Mixed (CTDE) & 4 insufficient advisors with different quality & 8\\ [1ex] 
 \hline
 \hline
\end{tabular}
\caption{Description of experimental settings}
\label{table:description}
\end{table}

\subsection{Pommerman} 

In Pommerman, the complete set of skills needed to be learned in order to win games include 1) escaping from the enemy, 2) obtaining power ups (bombs/life), 3) killing the enemy, 4) blasting walls to open routes, and 5) coordinating with a teammate (in the team version) \citep{resnick2018pommerman}. 

In our experiments, we consider two Pommerman domains. The first four experiments use the two-agent version of Pommerman and the fifth experiment uses the four-agent team version of Pommerman. Each episode in our training and execution experiments corresponds to a full Pommerman game with a randomized board and a maximum of 800 steps before completion. The game ends either when all the steps are completed or when one of the two Pommerman agents dies (two-agent version). In the team version, the game ends either when all the steps are complete (800 steps maximum) or when one of the two teams completely dies.

The first experiment (Experiment~1) uses four sufficient advisors of varying quality. The first advisor (Advisor~1) can teach all the strategies (i.e., various Pommerman skills as mentioned in Section~\ref{sec:experiments}) needed to win the Pommerman game. Advisor~2 can teach moves associated with killing the opponent if the opponent is very close to the agent and defensive strategies that help avoid the enemy. The third advisor (Advisor~3) can only teach defensive strategies that help in escaping the enemy, and cannot teach aggressive strategies needed to kill the enemy. Finally, the last advisor (Advisor~4) only provides random actions. Hence, all the agents should follow the first advisor as far as possible, and the fourth advisor must be avoided entirely.

In the second experiment (Experiment~2), we use a new set of four advisors. Here, the first advisor teaches only defensive skills (escaping the enemy), the second advisor can only teach aggressive skills (killing the enemy), the third advisor can only teach strategies that enable obtaining the power ups, and the fourth advisor teaches ways to seek and blast open wooden walls which opens up various paths in the game. It can be seen here that no one advisor is can teach all the strategies needed to win in Pommerman. However, together, all four advisors can teach the requisite strategies, and they need to be leveraged appropriately.

In Experiment~3, we use a set of four advisors of decreasing quality as in the first experiment; however, none of the advisors are can teach strategies that seek enemies and kill them (insufficient set). Also, none of the advisors are capable of teaching the skills needed to seek wooden walls to blast open. Hence, this set of advisors is insufficient for winning the Pommerman game. The first advisor (Advisor~1) can teach strategies to escape from an enemy, kill the enemy if the enemy is right next to the agent, and obtaining the power-ups. The second advisor (Advisor~2) can only teach strategies that help in escaping the enemy or killing an enemy close to the agent. Advisor~3 can only teach strategies that pertain to escaping the enemy, and Advisor~4 only provides random suggestions. 

In Experiment~4, we have a set of advisors similar to the second experiment; however, the advisors are incapable of teaching sufficient skills needed to win in Pommerman. The first advisor (Advisor~1) teaches only defensive skills to escape an enemy. The second advisor (Advisor~2) helps in learning a strategy that can kill an enemy right next to the agent. The third advisor (Advisor~3) can teach strategies that lead to obtaining the power-ups, and the fourth advisor (Advisor~4) can teach skills needed to blast open wooden walls, if the agent is very close to the wall. 

The fifth experiment (Experiment~5) with the team domain uses the same set of advisors as Experiment~1. 

The Pommerman environment was released by Resnick et al.~\cite{resnick2018pommerman} under the Apache2 license. 

\subsection{Pursuit domain}

The pursuit domain was first introduced by Gupta et al.~\cite{gupta2017cooperative}, and we use the implementation provided by the Petting Zoo environment \citep{terry2020pettingzoo} (released under MIT license). The game has a set of 8 pursuer agents cooperating with each other to capture a set of 30 evaders in the environment. We use the same reward function and environmental parameters as in Terry et al.~\citep{terry2020pettingzoo} with some minor modifications. In our setting, the agents get a reward of +1 for hitting (tagging) the evaders and a reward of +30 for catching an evader. An evader is decided to be caught if it is surrounded by a group of at least two pursuer agents. The captured evaders are removed from the environment. All agents get an urgency penalty of -0.1 at each time step of the game. Each episode has a maximum of 500 steps. The episode terminates either when all the evaders are captured or when all the steps are completed. All the pursuers receive the same reward at all time steps (global reward structure), where the rewards are distributed amongst all agents. Each pursuer observes a $7 \times 7$ grid around itself (the entire grid is $16 \times 16$), which means that the pursuer can get full information about other pursuers and evaders within this observable grid and no information outside this grid.

\subsection{Multi Particle Environment (Predator-Prey)}

The Multi Particle Environment (MPE) was originally released by Lowe et al.~\cite{lowe2017multi} as a set of testbeds for the purpose of testing algorithms that pertain to cooperative and mixed cooperative-competitive settings with characteristics of communication and obstacle interactions. From this suite of testbeds we use the Simple Tag environment that pertains to a Predator-Prey setting where a set of three predators try to capture a single prey. We use the environment defined as a part of the Petting Zoo library \cite{terry2020pettingzoo} (released under MIT license).  Here all agents have a discrete action space with a set of 5 actions (four cardinal directions and one action that signifies no movement). Both the predators and prey have a continuous observation space that corresponds to the velocity and position of all agents including the agent itself. The predators get a reward of +10 for hitting (colliding/tagging) prey and the prey get a punishment of -10 for being hit by any predator (the prey is not removed from the environment). The prey also receives a small additional penalty for exiting the field of play (see Terry et al.~\cite{terry2020pettingzoo} for more details). All the predators receive the same reward at all time steps (global reward structure), where the rewards are distributed amongst all predators. Our environment contains eight predators and eight prey, in addition to five obstacles that block the path of the predators and prey. Each game contains 500 steps of training or execution. We model this domain as a CTDE setting, where the actions taken, and the rewards obtained by all agents are available to each agent during training, but not available during execution.

\section{Hyperparameters And Implementation Details}\label{sec:hyperparameters}

All the hyperparameters in our implementation of baselines are either the same or closely match the values recommended by the respective papers that introduced these algorithms. Our algorithms also use similar hyperparameters as the baseline algorithms, with a few exceptions (for performance and computational efficiency reasons). 

The DQN \citep{mnih2015human}, CHAT \citep{wang2017improving}, TLQL \citep{li2019two}, ADMIRAL-DM \citep{Subramanian2022multiagent}, and MA-TLQL implementations use almost the same hyperparameters. These algorithms use a learning rate of 0.01, a discount factor of 0.9, a replay memory size of $2 \times 10^6$, and a fixed exploration rate of 0.9. The target network is replaced every 10 learning iterations using the hard replacement strategy. The evaluation and target networks use 3 fully connected layers (2 ReLU layers of 50 neurons and an output layer). We use a batch size of 32.

For the CHAT implementation, we use the neural network based confidence variant (NNHAT) from Wang and Taylor~\cite{wang2017improving}. We set a confidence threshold of 0.6 and use 3 fully connected layers (2 layers using ReLU as the activation function with 50 neurons and an additional output layer). The advisors are directly used in CHAT instead of preparing decision based rules from classifier models as done in \cite{wang2017improving} due to performance reasons and also because the advisors used in our experiments are either rule-based agents or pretrained networks and not human advisors as designed in CHAT. Since these advisors are considered to be extracted rules in CHAT (not actual demonstrators/advisors), we allow CHAT dependence on advisors in the execution phase as well. 

We use the PPR technique in TLQL for our experiments, though this is not used in Li et al.~\cite{li2019two}. This is done due to two reasons. First, unlike in single-agent settings, independent $Q$-learning methods do not have a policy improvement guarantee in multi-agent settings (as discussed in Section~\ref{sec:extending}) \cite{tan1993multi}, hence, the vanilla TLQL is not guaranteed to stop depending on advisors as motivated by Li et al.~\cite{li2019two}. Second, without PPR, the experimental TLQL training performance is very good (since it has unlimited dependence on advisors), while the execution performances are very poor (since advisors are not available during execution). The best execution performances for TLQL is obtained while using PPR. 

Regarding the implementation of the PPR technique, the TLQL, ADMIRAL-DM, MA-TLAC, and MA-TLQL implementations start with a value of $\epsilon' = 1$, which is linearly decayed to 0 at the end of training. There is no influence of advisors during execution, and hence, $\epsilon' = 0$ during execution.

To stay consistent with the description in Li et al.~\cite{li2019two}, the TLQL implementation uses a total of 3 networks (evaluation and target networks for low-$Q$ in addition to high-$Q$). The high-$Q$ and low-$Q$ networks use fully connected layers with the same architecture as described for the DQN. The MA-TLQL implementation uses four networks (two evaluation and target networks for the low-$Q$ and high-$Q$ respectively), also having the same configuration as described for the DQN. Further, the MA-TLQL uses a second replay buffer for the high-$Q$ in addition to the replay buffer of the low-$Q$. Both buffers use the same memory size of $2 \times 10^6$.

Regarding DQfD \citep{hester2018deep}, we set $1 \times 10^6$ as the demo buffer size and perform 50,000 mini-batch updates for pretraining. The replay buffer size is twice the size of the demo buffer. The N-step return weight is 1.0, the supervised loss weight is 1.0 and the L2 regularization weight is $10^{-5}$. The epsilon greedy exploration is 0.9. The discount factor is 0.99 and the learning rate is 0.002. The network architecture uses 3 fully connected layers (2 ReLU layers of 24 neurons and an output layer). The pretraining for DQfD comes from a data buffer related to a series of games where the advisors compete against each other.

The MA-TLAC uses two actor networks and two critic networks. The network architecture is the same as described for MA-TLQL. The actor networks use a learning rate of $10^{-6}$, and the critic networks use a learning rate of $10^{-3}$. 

For all training experiments, we use a set of 30 random seeds (1 -- 30). We use a new set of 30 random seeds (31 -- 60) for the execution experiments.

\section{Wall Clock Times}\label{sec:wallclocktimes}

All the training for the experiments were conducted on a virtual machine having 2 Nvidia A100 GPUs with a GPU memory of 40 GB. The CPUs use the AMD EPYC processors with a memory of 125 GB. The Pommerman experiments took an average of 12 hours wall clock time to complete, and the Pursuit experiments took an average of 15 hours wall clock time to complete.

\end{document}